\documentclass[letterpaper]{article} 
\usepackage[preprint]{aaai2027}  
\usepackage{times}  
\usepackage{helvet}  
\usepackage{courier}  
\usepackage[hyphens]{url}  
\usepackage{graphicx} 
\usepackage{natbib}  
\usepackage{booktabs}
\usepackage{multirow}
\usepackage{makecell}
\usepackage{array}
\usepackage{caption} 
\usepackage{amsmath} 
\usepackage{amsthm}
\usepackage{amsfonts} 
\usepackage{amssymb}
\usepackage{rotating}
\usepackage{tabularx}

\newtheorem{proposition}{Proposition}
\newtheorem{theorem}{Theorem}
\usepackage{algorithm}
\usepackage{algorithmic}

\usepackage{newfloat}
\usepackage{listings}
\DeclareCaptionStyle{ruled}{labelfont=normalfont,labelsep=colon,strut=off} 
\floatstyle{ruled}
\newfloat{listing}{tb}{lst}{}
\floatname{listing}{Listing}
\title{MeClear: Cooperative Game-Theoretic Attribution and Risk-Aware Memory Clearance for Long-Horizon LLM Agents{}}
\author{
    Boyu Yang\equalcontrib\textsuperscript{\rm 1},
    Jiazheng Sun\equalcontrib\textsuperscript{\rm 1},
    Zilong Lu\textsuperscript{\rm 1},
    Zhi Qiu\textsuperscript{\rm 2},
    Xin Peng\textsuperscript{\rm 1},
    Jun Zheng\textsuperscript{\rm 2},
}

\affiliations{
    \textsuperscript{\rm 1}College of Computer Science and Artificial Intelligence, Fudan University,
    Shanghai 200433, China\\
    \textsuperscript{\rm 2}School of Cyberspace Science and Technology, Beijing Institute of Technology, 
    Beijing 100081, China
}

\usepackage{bibentry}

\begin{document}

\maketitle

\begin{abstract}
Long horizon Large Language Model (LLM) agents rely on external memory systems to preserve user preferences and task knowledge across extended interactions. Conventional retrieval mechanisms optimize semantic compatibility rather than downstream utility, frequently introducing outdated, misleading, or conflicting evidence into the active context. We present MeClear, a task conditioned memory clearance framework that identifies memories featuring negative downstream utility through cooperative attribution and selectively suppresses them from agent execution. MeClear combines Leave One Out screening with sampled cooperative Shapley attribution to distribute utility across interacting evidence, effectively resolving redundant conflict masking where single removal evaluations fail. Utilizing attribution rankings, MeClear executes a query scoped minimal clearance strategy over a nested filtration, verifying task recovery on the cleared context without permanently altering the persistent memory bank. Comprehensive experimental evaluations across ten long dialogue memory pools demonstrate that MeClear achieves a target recall of 85.9\% and an overall task recovery rate of 82.3\%, representing a 25.5 percentage point improvement over Leave One Out (LOO) baselines.
\end{abstract}

\begin{links}
    \link{Code}{https://github.com/FudanSELab/MeClear}
\end{links}


\section{Introduction}

Long-horizon Large Language Model (LLM) agents increasingly depend on external memory to retain preferences, observations, and state across extended interactions \cite{wang2024survey,zhang2025memorysurvey}. Early architectures demonstrated that persistent experience supports planning, reflection, and continual adaptation \cite{park2023generative}. However, growing interaction histories accumulate stale, misleading, redundant, or incompatible records. Recent studies reveal that inaccurate experiences propagate errors across future tasks \cite{xiong2025memory}, while semantically related memories often prove contextually inappropriate or functionally harmful \cite{zhang2026beyond,ha2026memguard}. These findings expose a fundamental failure of relevance-centered retrieval: high semantic similarity does not ensure downstream task utility. Reliable memory deployment therefore demands task-conditioned, post-retrieval evaluation of how retrieved records causally influence agent execution.

As illustrated in Figure~\ref{fig:motivation}, jointly processed memories exhibit complex dependencies such as redundancy and joint harm, causing isolated record evaluations to fail. While recent frameworks refine credit assignment via evidence-anchored rewards \cite{ma2026finemem} or single-record causal interventions \cite{srivastava2026causal}, single-record deletion collapses when redundant memories independently sustain task failure, yielding zero observable counterfactual change. Coalition-based attribution resolves this local masking by measuring marginal contributions across subset permutations \cite{ghorbani2019datashapley,jia2019efficient,nematov2026source}. Because exact coalition enumeration is computationally prohibitive for multi-record context windows, practical clearance must efficiently approximate cooperative interaction effects while preserving sufficient contextual variation to uncover masked toxicity.

Directly modifying persistent memory introduces severe safety risks due to the inherently query-conditioned nature of memory utility across long-horizon deployment. A stored record that degrades performance on a query can provide indispensable contextual grounding for subsequent tasks, rendering permanent database overwrites prone to catastrophic cross-task performance regression \cite{wang2024wise,xu2025memory}. Retrieval-time defenses mitigate this danger by dynamically regulating active context visibility during inference rather than altering database records \cite{ha2026memguard,zhang2026beyond}. Nevertheless, existing admission filters fail to isolate which interacting memories causally induce downstream task failures, nor do they verify whether suppressing records restores execution accuracy. Reliable memory maintenance consequently demands reversible, minimal context clearance accompanied by explicit, post-suppression behavioral recovery verification.

\begin{figure*}[t]
    \includegraphics[width=\textwidth]{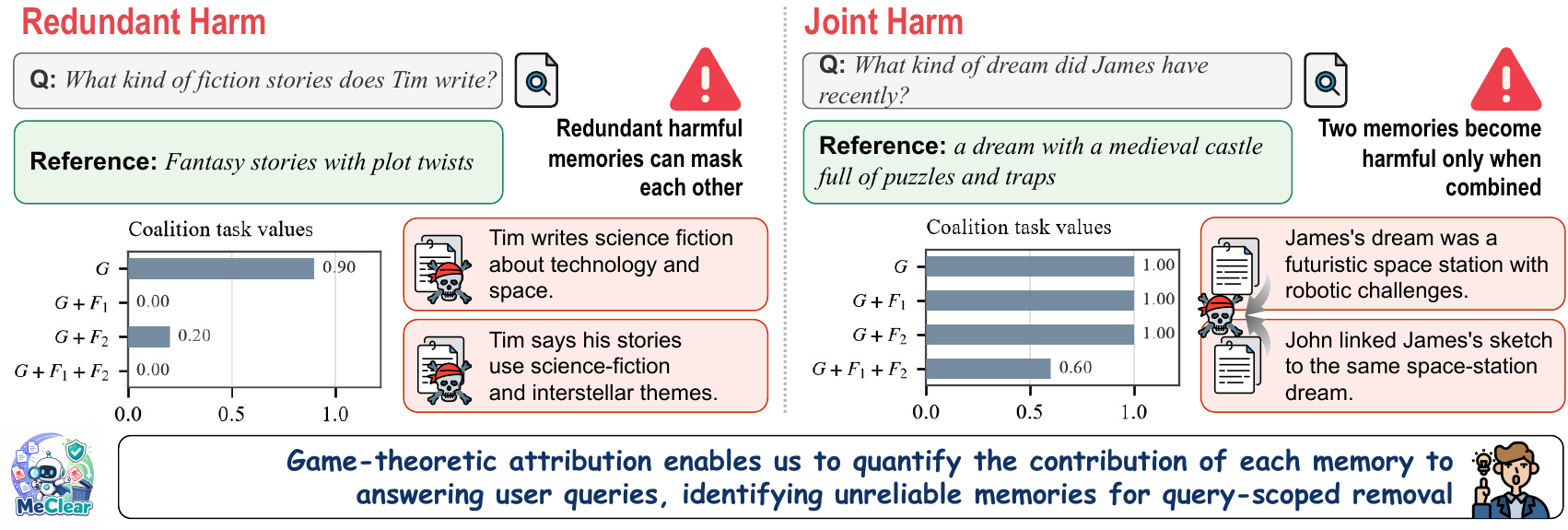}
    \caption{
   Motivation of MeRepair. Semantic relevance alone cannot determine whether retrieved memories are useful for future tasks. Historical utility may hide temporal degradation and memory interactions. MeRepair estimates future memory utility and performs uncertainty-aware
    repair to maintain reliable long-term memory for LLM agents.
    }
    \label{fig:motivation}
\end{figure*}


To address this problem, we present MeClear, a cooperative game-theoretic attribution and risk aware memory clearance framework for long horizon Large Language Model agents. MeClear formulates memory maintenance as a task conditioned decision process connecting counterfactual screening, cooperative Shapley attribution, and verified context clearance. The framework first isolates active contexts from retrieval stochasticity and applies Leave One Out (LOO) screening to filter strong positive evidence. It subsequently executes sampled cooperative attribution to quantify individual memory contributions across diverse coalitions, resolving redundant conflict masking where local deletion fails. Based on attribution priorities, MeClear executes a query scoped minimal clearance strategy over a nested candidate filtration, verifying task recovery prior to output generation without permanently altering the underlying memory bank. To the best of our knowledge, we are the first to leverage cooperative game-theoretic allocation to address harmful memory context clearance in LLM agents. 
Our contributions are summarized as follows:

\begin{itemize}
    \item We propose MeClear, a task conditioned memory clearance framework for long horizon Large Language Model agents, unifying cooperative game attribution, structural interaction diagnosis, and verified risk aware context clearance.

    \item We develop a cooperative memory attribution methodology integrating Leave One Out screening with sampled Shapley estimation, effectively resolving redundant conflict masking and modeling multi memory interactions.

    \item We conduct comprehensive evaluations across ten long dialogue memory pools, demonstrating that MeClear achieves an injected target recall of 85.9\% and an overall task recovery rate of 82.3\%, representing a 25.5 percentage point improvement over Leave One Out baselines.
\end{itemize}




\section{Related Work}
\label{sec:related_work}

\subsubsection{Memory Management and Operations for LLM Agents}
\label{subsubsec:agent_memory_systems}

External memory enables LLM agents to retain information across extended interactions. Early paradigms rely on heuristic retrieval \cite{park2023generative}, hierarchical context \cite{packer2023memgpt}, or verbal reflection \cite{shinn2023reflexion}, whereas recent frameworks optimize structured memory operations via graph representations, offline consolidation, or reinforcement learning \cite{chhikara2025mem0,zhang2026lightmem,wu2026gam,yu2026agentic,yan2026memoryr1}. However, superior retrieval does not guarantee downstream task utility. Semantic relevance frequently retrieves stale, redundant, or functionally incompatible records \cite{zhang2026beyond,ha2026memguard,liu2026worldmemarena}. Rather than relying solely on semantic similarity or admission filtering, MeClear evaluates post-retrieval functional effects and contextual interactions to dynamically regulate memory visibility.

\subsubsection{Long-Term Memory Evaluation and Attribution}
\label{subsubsec:memory_evaluation_attribution}

Long-term memory benchmarks assess overall response quality across multi-session dialogues and reasoning \cite{maharana2024locomo,wu2024longmemeval,li2026locomoplus}, but offer limited causal credit assignment for individual records. While operation-level rewards \cite{ma2026finemem} and single-record counterfactual interventions \cite{roy2025evidence,srivastava2026causal} attempt fine-grained attribution, they miss interaction-dependent phenomena such as mutual redundancy or joint toxicity. Although Shapley-based valuations capture coalition-level marginal contributions \cite{ghorbani2019datashapley,jia2019efficient,nematov2026source}, applying them to agentic context remains unexplored. MeClear adapts coalition-aware attribution to persistent memory, unifying interaction diagnostics with behaviorally verified context clearance under bounded computational budgets.

\subsubsection{Memory Editing and Safe Context Clearance}
\label{subsubsec:memory_editing_repair}

Model editing updates factual representations via parameter shifts or external patching \cite{meng2022rome,meng2023memit,mitchell2021mend,mitchell2021fast,wang2024wise}. In agent memory, record utility is dynamic and query-dependent; permanent modification risks propagating irreversible errors across future tasks \cite{yu2026agentic,yan2026memoryr1,xu2025memory}. Retrieval-time filtering avoids database overwrites \cite{ha2026memguard,zhang2026beyond} but fails to verify interaction-dependent behavioral consequences before suppressing context. MeClear bridges this gap by unifying reversible visibility control, counterfactual contribution analysis, and behavioral verification, suppressing harmful memory influences on a per-query basis while leaving the underlying database intact.

\section{Problem Formulation}
\label{sec:problem}

Consider a long-horizon Large Language Model (LLM) agent that continuously accumulates information across extended interactions. At task step $t$, the agent maintains an external memory bank $\mathcal{B}_t$ containing historical observations, factual evidence, and evolving user-specific information. Given a query $q_t$, an existing retrieval mechanism $\rho$ returns a bounded execution context

\begin{equation}
\mathcal{M}_t
=
\rho(q_t,\mathcal{B}_t;K)
=
\{m_1,m_2,\ldots,m_K\}
\subseteq
\mathcal{B}_t ,
\label{eq:retrieval_subset}
\end{equation}

where $K$ is the retrieval budget determined by the underlying memory system and the available context capacity. Each memory record is represented as $m_i=(x_i,\xi_i)$, where $x_i$ denotes its textual content and $\xi_i$ contains auxiliary information such as timestamp, source, confidence, or provenance. MeClear does not replace the underlying retriever. Instead, it operates on the already retrieved context $\mathcal{M}_t$ and evaluates whether each exposed memory actually benefits the current task. This separation is essential because semantic relevance does not guarantee downstream utility: a retrieved memory may be topically related to the query while being outdated, misleading, redundant, or incompatible with other active evidence.
Let $A_{\theta}$ denote the task agent parameterized by $\theta$. For any active memory coalition $S\subseteq\mathcal{M}_t$, the agent generates an output $y$ according to $p_{\theta}(y\mid q_t,S)$. We define the task-conditioned value of coalition $S$ as

\begin{equation}
v_t(S)
=
\mathbb{E}_{y\sim p_{\theta}(\cdot\mid q_t,S)}
\left[
r_t(y)
\right],
\qquad
v_t:2^{\mathcal{M}_t}\rightarrow[0,1],
\label{eq:task_value_function}
\end{equation}
where $r_t(y)\in[0,1]$ measures task correctness, and $v_t(S)$ denotes the query-conditioned expected utility under memory subset $S$, empirically estimated as $\widehat v_t(S)$ via agent execution. Because retrieved memories interact non-independently through redundancy or complementarity, we formulate $(\mathcal{M}_t,v_t)$ as a cooperative game with memories as players and $v_t$ as the characteristic function. The cooperative contribution of memory $m_i$ is defined as

\begin{equation}
\begin{aligned}
&\psi_{t,i} = {}  \\
& \sum_{S\subseteq\mathcal{M}_t\setminus\{m_i\}}
\frac{|S|!\left(K-|S|-1\right)!}{K!}
\left[
v_t(S\cup\{m_i\})-v_t(S)
\right].
\end{aligned}
\label{eq:memory_contribution}
\end{equation}

The coefficient in Equation~\eqref{eq:memory_contribution} equals the probability that $S$ forms the predecessor coalition of $m_i$ under a uniformly random ordering of the $K$ memories. Thus, $\psi_{t,i}$ measures the marginal effect of $m_i$ across diverse contextual coalitions rather than only around the complete retrieved context. The resulting allocation satisfies $\sum_{i=1}^{K}\psi_{t,i} = v_t(\mathcal{M}_t)-v_t(\emptyset)$. A positive contribution indicates that the memory improves task utility on average, whereas a negative value indicates that its presence decreases task utility across coalition contexts.
To separate meaningful harm from negligible negative variation, we introduce a single tolerance $\tau\geq0$ and define the query-conditioned harmful memory set as

\begin{equation}
\mathcal{H}_t
=
\left\{
m_i\in\mathcal{M}_t
\;\middle|\;
\psi_{t,i}<-\tau
\right\}.
\label{eq:true_harmful_set}
\end{equation}

Membership in $\mathcal{H}_t$ is not an intrinsic or permanent property of a memory record. It is determined jointly by the current query $q_t$, the retrieved context $\mathcal{M}_t$, the task agent $A_{\theta}$, and the task-value function $v_t$. MeClear therefore treats harmful-memory attribution as a query-conditioned visibility decision rather than an irreversible modification of the persistent memory bank.
For any candidate clearance set $\mathcal{C}\subseteq\mathcal{H}_t$, clearing $\mathcal{C}$ produces the active context $\mathcal{M}_t\setminus\mathcal{C}$. The corresponding task gain is defined as $g_t(\mathcal{C}) = v_t(\mathcal{M}_t\setminus\mathcal{C}) - v_t(\mathcal{M}_t)$. A desirable intervention should first maximize task recovery and then, among all interventions attaining the same recovery, remove the fewest memories. We formulate this parameter-free lexicographic objective as

\begin{equation}
G_t = \max_{\mathcal{C}\subseteq\mathcal{H}_t} g_t(\mathcal{C}),
\qquad
\mathcal{C}_t \in \operatorname*{arg\,min}_{\mathcal{C}\subseteq\mathcal{H}_t} \left\{ |\mathcal{C}| \;\middle|\; g_t(\mathcal{C}) = G_t \right\},
\label{eq:clearance_objective}
\end{equation}

where $G_t$ determines the maximum task gain attainable by clearing a subset of harmful memories, while $\mathcal{C}_t$ selects a minimum-cardinality intervention among all clearance sets attaining this gain. The formulation implements the minimal-intervention principle without introducing an additional weighting coefficient.
Because $\emptyset\subseteq\mathcal{H}_t$ and $g_t(\emptyset)=0$, the optimal gain always satisfies $G_t\geq0$. If no nonempty subset improves the task, the minimum-cardinality solution is $\mathcal{C}_t=\emptyset$. The exact contribution values and task gains are generally unavailable during execution. MeClear therefore approximates them through finite counterfactual evaluations and verifies the selected intervention on the current query.

\begin{figure*}[t]
  \centering
  \includegraphics[width=\textwidth]{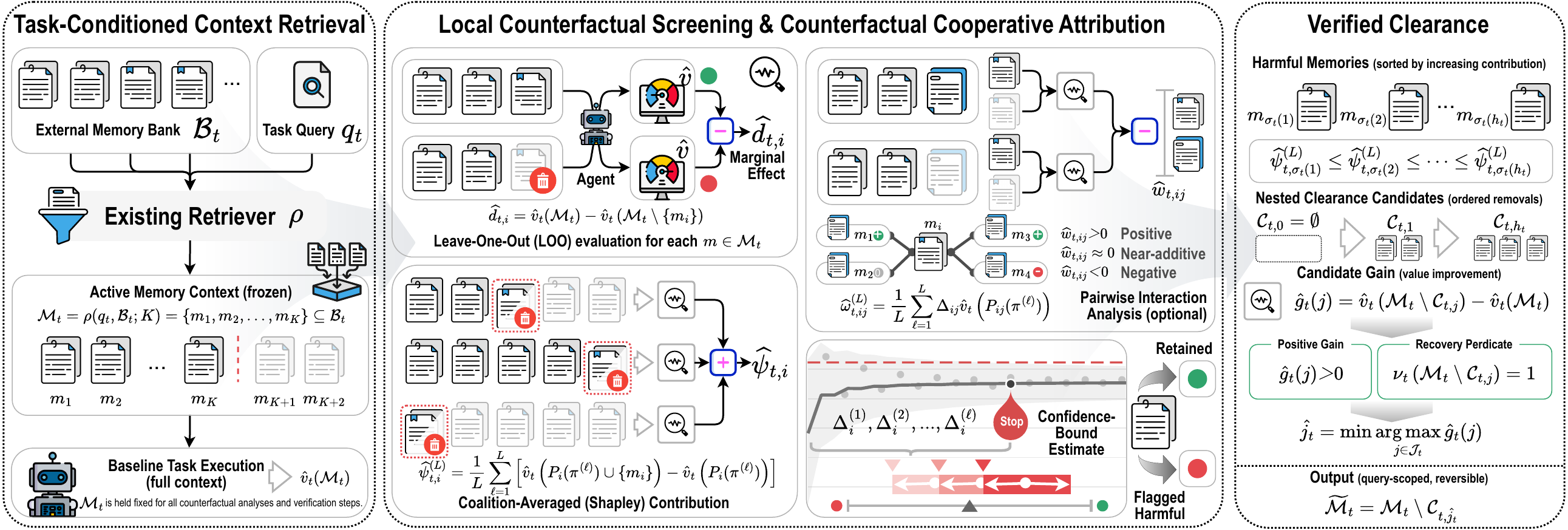}
\caption{Overall architecture of MeClear, combining local screening and cooperative Shapley attribution to output a query-scoped cleared context without modifying persistent storage.}
  \label{fig:maclear_framework}
\end{figure*}

\section{Method}
\label{sec:method}

As shown in Figure~\ref{fig:maclear_framework}, MeClear selectively suppresses harmful memories from the active context without modifying persistent storage. It operates across four stages: adaptive retrieval, local screening, cooperative attribution, and verified clearance. The architecture introduces no stage specific tuning coefficients, governed transparently by context capacity, permutation budget, and a single tolerance threshold.

\subsubsection{Adaptive Task-Conditioned Context Retrieval}
\label{subsubsec:context_retrieval}

For query $q_t$, the external memory system retrieves an active context $\mathcal{M}_t$. The term adaptive denotes the query-dependent nature of context selection rather than a re-trained retrieval model, as MeClear operates as a post-retrieval layer without modifying the parameters or indexing of $\rho$. This separation prevents semantic similarity from being conflated with downstream utility, as the retriever estimates query-record compatibility whereas $v_t(S)$ evaluates agent behavior under exposed coalition $S$. Consequently, memories with comparable retrieval scores can produce distinct task outcomes due to temporal drift, factual conflicts, redundancy, or contextual interaction:
\begin{equation}
\mathcal{M}_t = \rho(q_t, \mathcal{B}_t; K).
\label{eq:active_context_retrieval}
\end{equation}

To isolate the behavioral impact of memory suppression from retrieval volatility, MeClear freezes $\mathcal{M}_t$ prior to attribution, executing all subsequent interventions on subsets of this fixed context without re-invoking the retriever. For each coalition $S\subseteq\mathcal{M}_t$, the empirical value $\widehat v_t(S)$ is obtained by executing the task agent on $S$ under a fixed evaluator and cached by memory identity. Evaluating the complete context value $\widehat v_t(\mathcal{M}_t)$ first establishes a shared reference baseline for all subsequent counterfactual comparisons.

\subsubsection{Local Counterfactual Screening}
\label{subsubsec:loo_screening}

MeClear first establishes a local contribution profile using Leave-One-Out (LOO) counterfactual interventions. For each retrieved record $m_i\in\mathcal{M}_t$, the empirical local effect is defined as:
\begin{equation}
\widehat d_{t,i}
=
\widehat v_t(\mathcal{M}_t)
-
\widehat v_t
\left(
\mathcal{M}_t\setminus\{m_i\}
\right).
\label{eq:local_effect}
\end{equation}
A negative value of $\widehat d_{t,i}$ indicates that suppressing $m_i$ improves execution performance on the complete context, whereas a positive value reflects a reduction in task utility upon removal. Values approaching zero remain inconclusive because auxiliary or substitutable evidence may mask the underlying effect. Consequently, MeClear utilizes LOO as an interpretable local screening signal rather than a definitive harmfulness criterion.
The structural limitation of single-record deletion becomes particularly pronounced under redundant harmful evidence, where multiple records independently induce task degradation.

\begin{theorem}[Redundant-Harm Blind Spot]
\label{thm:redundant_harm}
Let $m_i$ and $m_j$ be substitutable harmful memories whose joint effect on the task value function satisfies:
\begin{equation}
\begin{aligned}
v_t(S) = {} & u_t\left(S\setminus\{m_i,m_j\}\right) \\
& - \Delta_t \mathbf{1}\left[S\cap\{m_i,m_j\}\neq\emptyset\right], \quad \Delta_t>0,
\end{aligned}
\label{eq:redundant_game}
\end{equation}
where $u_t$ is independent of $m_i$ and $m_j$. If both records are present in $\mathcal{M}_t$, then:
\begin{equation}
d_{t,i} = d_{t,j} = 0,
\qquad
\psi_{t,i} = \psi_{t,j} = -\frac{\Delta_t}{2}.
\label{eq:redundant_harm_result}
\end{equation}
\end{theorem}

Deleting either record leaves the other active, maintaining the degradation penalty in Equation~\eqref{eq:redundant_game} and yielding zero LOO effect. Under random permutations, $m_i$ precedes $m_j$ with probability $1/2$, yielding expected contribution $\psi_{t,i} = -\Delta_t/2$ as shown in Appendix A. Theorem~\ref{thm:redundant_harm} demonstrates that a zero LOO effect does not imply zero cooperative harm. While LOO measures utility only locally, cooperative attribution evaluates marginal contributions across coalitions where redundant substitutes are absent. MeClear therefore retains the local profile as a diagnostic reference while anchoring harmful memory selection on coalition-averaged contributions.

\begin{figure*}[t]
    \centering
    \includegraphics[width=1\linewidth]{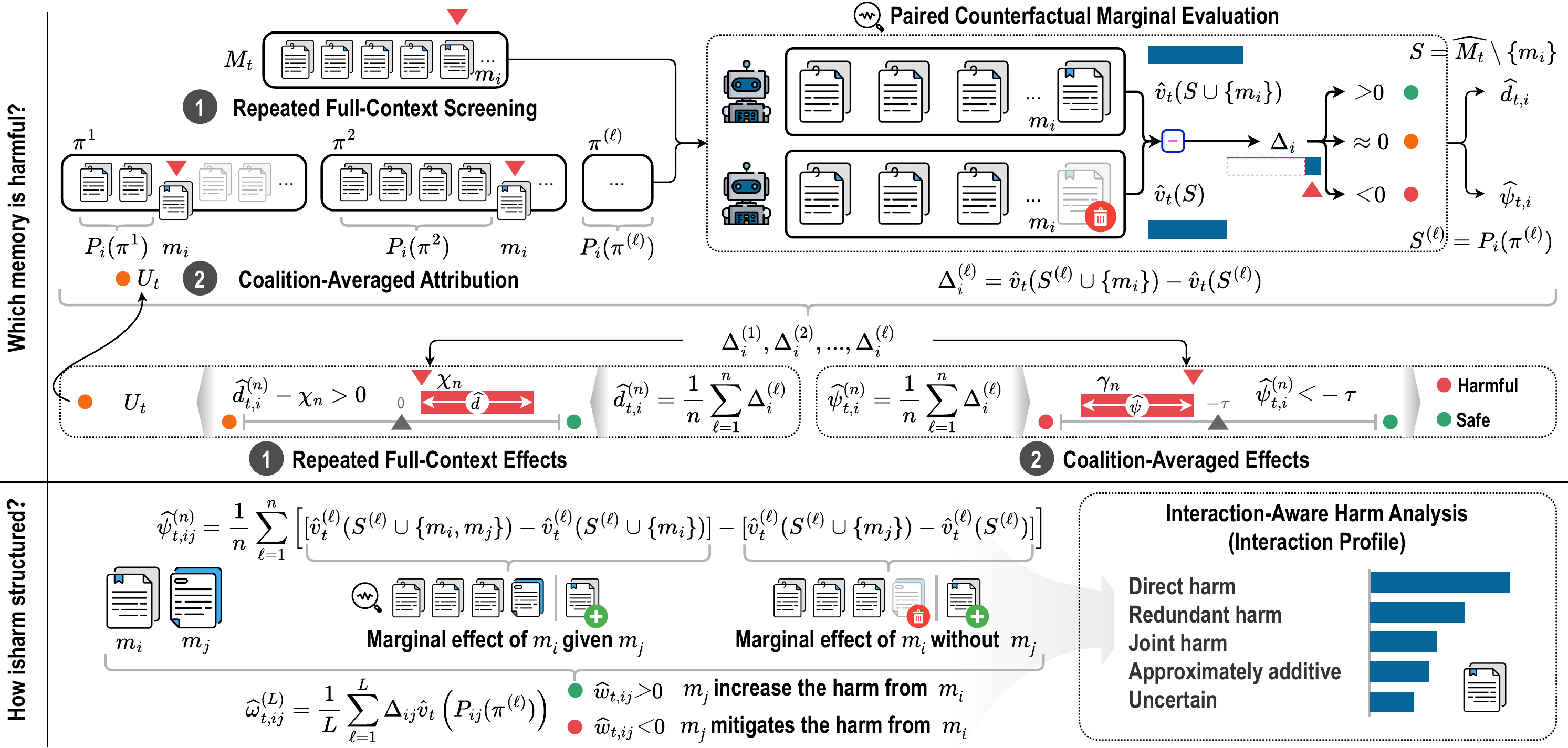}
    \caption{Counterfactual attribution and interaction profiling in MeClear. \textbf{Top:} Individual memory harm attribution. \textbf{Bottom:} Pairwise interaction analysis diagnosing direct, redundant, and joint memory harm dependencies.}
    \label{fig:harm_attribution_analysis}
\end{figure*}

\subsubsection{Counterfactual Cooperative Attribution}
\label{subsubsec:cooperative_attribution}

MeClear next estimates memory contributions through the equivalent permutation representation. As depicted in Figure~\ref{fig:harm_attribution_analysis}, individual memory utility is quantified via permutation sampled Shapley contribution estimation across contextual permutations, while pairwise interaction profiles diagnose structural dependencies among evidence. Let $\Pi_t$ denote the set of all permutations of $\mathcal{M}_t$. For $\pi\in\Pi_t$, let $P_i(\pi)$ contain the memories appearing before $m_i$. The exact contribution is:
\begin{equation}
\psi_{t,i}
=
\frac{1}{|\Pi_t|}
\sum_{\pi\in\Pi_t}
\left[
v_t
\left(
P_i(\pi)\cup\{m_i\}
\right)
-
v_t
\left(
P_i(\pi)
\right)
\right].
\label{eq:permutation_contribution}
\end{equation}
Unlike single-record screening, Equation~\eqref{eq:permutation_contribution} evaluates the marginal contribution of $m_i$ across diverse predecessor coalitions, exposing negative utility masked by substitution or redundancy in the full context.
Because enumerating all permutations is computationally infeasible for non-trivial context size $K$, MeClear draws $L$ independent permutations $\pi^{(1)},\ldots,\pi^{(L)}$ and computes the sample-average estimate:
\begin{equation}
\widehat{\psi}_{t,i}^{(L)}
=
\frac{1}{L}
\sum_{\ell=1}^{L}
\left[
\widehat v_t
\left(
P_i(\pi^{(\ell)})\cup\{m_i\}
\right)
-
\widehat v_t
\left(
P_i(\pi^{(\ell)})
\right)
\right].
\label{eq:sampled_shapley}
\end{equation}
Each permutation is sampled over the context $\mathcal{M}_t$, ensuring beneficial, neutral, and harmful memories co-occur in predecessor coalitions to preserve the underlying game structure.

\begin{proposition}[Finite-Sample Cooperative Estimation]
\label{prop:finite_sample_attribution}
For independent uniformly sampled permutations, the estimator in Equation~\eqref{eq:sampled_shapley} is unbiased:
\begin{equation}
\mathbb{E}
\left[
\widehat{\psi}_{t,i}^{(L)}
\right]
=
\widetilde{\psi}_{t,i},
\label{eq:sampled_unbiasedness}
\end{equation}
and for any $\varepsilon>0$, satisfies the concentration bound:
\begin{equation}
\Pr
\left(
\max_{1\leq i\leq K}
\left|
\widehat{\psi}_{t,i}^{(L)}
-
\widetilde{\psi}_{t,i}
\right|
\geq
\varepsilon
\right)
\leq
2K
\exp
\left(
-\frac{L\varepsilon^2}{2}
\right).
\label{eq:sampled_uniform_bound}
\end{equation}
\end{proposition}

The summands in Equation~\eqref{eq:sampled_shapley} represent independent bounded observations of marginal contribution. Applying Hoeffding's inequality with a union bound across all $K$ records yields the concentration result in Appendix B. Proposition~\ref{prop:finite_sample_attribution} characterizes the approximation error introduced by finite sampling for a fixed empirical game without assuming the underlying model is an unbiased population estimator.
MeClear constructs the operational harmful set by thresholding estimated contributions with tolerance $\tau$:
\begin{equation}
\widehat{\mathcal{H}}_t
=
\left\{
m_i\in\mathcal{M}_t
\;\middle|\;
\widehat{\psi}_{t,i}^{(L)}<-\tau
\right\}.
\label{eq:operational_harmful_set}
\end{equation}

Using the unified tolerance $\tau$ avoids introducing additional hyperparameter thresholds into attribution.
To characterize higher-order structural dependencies, MeClear further estimates pairwise non-additivity. Let $\Delta_{ij}\widehat v_t(S) = \widehat v_t(S\cup\{m_i,m_j\}) - \widehat v_t(S\cup\{m_i\}) - \widehat v_t(S\cup\{m_j\}) + \widehat v_t(S)$. For predecessor set $P_{ij}(\pi)$ preceding both $m_i$ and $m_j$, the sampled interaction effect is:
\begin{equation}
\widehat{\omega}_{t,ij}^{(L)}
=
\frac{1}{L}
\sum_{\ell=1}^{L}
\Delta_{ij}\widehat v_t
\left(
P_{ij}(\pi^{(\ell)})
\right),
\qquad
i\neq j .
\label{eq:sampled_interaction}
\end{equation}

The resulting interaction profile diagnoses non-additive structures such as substitution, complementarity, and joint interference, serving as a structural behavioral diagnostic rather than a semantic classifier.

\begin{table*}[t]
  \centering
  \small
  \setlength{\tabcolsep}{3.5pt}
  \renewcommand{\arraystretch}{1.15}

  \begin{minipage}{\textwidth}
    \textbf{Panel A. Case construction and causal signature}\par\smallskip

    \begin{tabularx}{\textwidth}{
      @{}
      c
      l
      >{\raggedright\arraybackslash}X
      >{\raggedright\arraybackslash}X
      >{\centering\arraybackslash}p{3.1cm}
      >{\raggedright\arraybackslash}X
      @{}
    }
      \toprule
      Case & Fault & QA $\rightarrow$ Gold & Injected pair $\mathcal{M}$
           & Causal signature $\mathbf{c}$ & Mechanism \\
      \midrule
      A & Temporal
        & John in Italy: which month? $\rightarrow$ Dec. 2023
        & Two paraphrases: Jan. 2023
        & $(1.00,\,0.50,\,0.50,\,0.50)$
        & Flat plateau; each fault masks the other's local effect \\
      \addlinespace[2pt]
      B & Conflict
        & Dogs' reaction to snow? $\rightarrow$ Confused
        & Two paraphrases: excited/thrilled
        & $(1.00,\,0.00,\,0.00,\,0.00)$
        & Each fault independently reaches the failure floor \\
      \addlinespace[2pt]
      C & Factual
        & Effect on Calvin's songs? $\rightarrow$ Fresh vibe
        & Two paraphrases: warm, vintage tone
        & $(1.00,\,0.00,\,0.00,\,0.00)$
        & Floor masking; LOO instead selects background $G_3$ \\
      \bottomrule
    \end{tabularx}

    \vspace{0.6em}
    \textbf{Panel B. Attribution and deletion outcome}\par\smallskip

    \begin{tabularx}{\textwidth}{
      @{}
      c
      c
      *{5}{>{\centering\arraybackslash}X}
      @{}
    }
      \toprule
      Case
        & \makecell{MeClear contribution \\
          $\widehat{\psi}(F_1),\widehat{\psi}(F_2)$}
        & LOO
        & MeClear
        & ContextCite
        & ProxySPEX
        & \makecell{LLM \\ baseline} \\
      \midrule
      A
        & $(-0.094,-0.219)$
        & \makecell{$\varnothing$ \\ 0.50 / No}
        & \makecell{\textbf{$\mathcal{M}$ (Exact)} \\ \textbf{1.00 / Yes}}
        & \makecell{$\varnothing$ \\ 0.50 / No}
        & \makecell{$\{F_1\}$ \\ 0.50 / No}
        & \makecell{$\varnothing$ \\ 0.50 / No} \\
      \addlinespace[1pt]
      B
        & $(-0.281,-0.156)$
        & \makecell{$\varnothing$ \\ 0.00 / No}
        & \makecell{\textbf{$\mathcal{M}$ (Exact)} \\ \textbf{1.00 / Yes}}
        & \makecell{$\{F_1\}$ \\ 0.10 / No}
        & \makecell{$\{F_1\}$ \\ 0.10 / No}
        & \makecell{$\mathcal{M}{\cup}\{G_1\}$ \\ 0.00 / No} \\
      \addlinespace[1pt]
      C
        & $(-0.206,-0.075)$
        & \makecell{$\{G_3\}$ \\ 0.10 / No}
        & \makecell{\textbf{$\mathcal{M}$ (Exact)} \\ \textbf{1.00 / Yes}}
        & \makecell{$\{F_2\}$ \\ 0.10 / No}
        & \makecell{$\varnothing$ \\ 0.00 / No}
        & \makecell{\textbf{$\mathcal{M}$ (Exact)} \\ \textbf{1.00 / Yes}} \\
      \bottomrule
    \end{tabularx}

    \par\smallskip
    {\footnotesize
  \textit{Notes.}
  $\mathbf{c}=(v(G),v(G{+}F_1),v(G{+}F_2),v(G{+}M))$, with frozen top-5 context $G$ and error set $\mathcal{M}=\{F_1,F_2\}$. Cells report deleted set $H$, post-deletion task value, and Recovery status (requiring 2/2 correct trials). Hyperparameters: MeClear $L=16$, $\kappa=\tau=0.05$; ContextCite and ProxySPEX $B=32$; LLM baseline score threshold $\geq 0.5$. Cases are qualitative examples from qualifying redundant scenarios.
}
  \end{minipage}
  \caption{Representative redundant-conflict cases illustrating how MeClear
  resolves local masking.}
  \label{tab:meclear-redundant-cases}
\end{table*}

\subsubsection{Verified Query-Scoped Clearance}
\label{subsubsec:verified_clearance}

While the operational harmful set identifies candidate negative memories, suppressing all negatively attributed records may eliminate more context than necessary for task recovery. MeClear therefore approximates the ideal objective in Equation~\eqref{eq:clearance_objective} through an ordered sequence of query-scoped clearance candidates. Let $h_t=|\widehat{\mathcal{H}}_t|$. The records in $\widehat{\mathcal{H}}_t$ are ordered from the most negative to the least negative estimated contribution:
\begin{equation}
\widehat{\psi}_{t,\sigma_t(1)}^{(L)}
\leq
\widehat{\psi}_{t,\sigma_t(2)}^{(L)}
\leq
\cdots
\leq
\widehat{\psi}_{t,\sigma_t(h_t)}^{(L)} ,
\label{eq:harmful_order}
\end{equation}
where $\sigma_t$ is a permutation of indices in $\widehat{\mathcal{H}}_t$. A smaller index indicates a stronger negative contribution and thus a higher priority for removal.
MeClear constructs a nested family of clearance candidates along this priority chain:
\begin{equation}
\mathcal{C}_{t,j}
=
\left\{
m_{\sigma_t(1)},
m_{\sigma_t(2)},
\ldots,
m_{\sigma_t(j)}
\right\},
\qquad
0\leq j\leq h_t,
\label{eq:clearance_chain}
\end{equation}
with baseline $\mathcal{C}_{t,0}=\emptyset$. The candidate sequence satisfies $\mathcal{C}_{t,j-1}\subseteq\mathcal{C}_{t,j}$ and $|\mathcal{C}_{t,j}|=j$, reducing the combinatorial search space from $2^{h_t}$ arbitrary subsets to $h_t+1$ candidate contexts.
For each candidate, MeClear computes the empirical query-scoped gain:
\begin{equation}
\widehat{g}_t(j)
=
\widehat{v}_t
\left(
\mathcal{M}_t\setminus\mathcal{C}_{t,j}
\right)
-
\widehat{v}_t(\mathcal{M}_t).
\label{eq:query_scoped_gain}
\end{equation}

Because all candidates are masked directly from the frozen context $\mathcal{M}_t$ without re-invoking retrieval, any observed behavioral gain is strictly attributable to memory removal.
Let $\nu_t(S)\in\{0,1\}$ denote a fixed task-recovery predicate for query $q_t$. The set of admissible candidate indices $\mathcal{J}_t \subseteq \{0, 1, \dots, h_t\}$ is defined as:
\begin{equation}
\begin{aligned}
\mathcal{J}_t = \{0\} \cup \Big\{ & j \in \{1, \dots, h_t\} \;\Big|\; \\
& \widehat{g}_t(j) > 0, \; \nu_t\left(\mathcal{M}_t \setminus \mathcal{C}_{t,j}\right) = 1 \Big\}.
\end{aligned}
\label{eq:admissible_set}
\end{equation}

Including $j=0$ ensures $\mathcal{J}_t$ is non-empty. MeClear selects the admissible candidate maximizing verified empirical gain with minimal intervention cardinality:
\begin{equation}
\widehat{j}_t
=
\min \operatorname*{arg\,max}_{j \in \mathcal{J}_t} \widehat{g}_t(j).
\label{eq:query_scoped_selection}
\end{equation}

The maximization isolates candidates with peak empirical recovery, while the outer minimum breaks ties by choosing the smallest index and minimal intervention cardinality. The resulting query-conditioned context is $\widetilde{\mathcal{M}}_t = \mathcal{M}_t \setminus \mathcal{C}_{t,\widehat{j}_t}$, which serves as a query-scoped active context without permanently modifying the persistent memory bank $\mathcal{B}_t$.

\begin{theorem}[Query-Scoped Clearance Guarantee]
\label{thm:query_scoped_clearance}
For any fixed empirical evaluator $\widehat{v}_t$ and nested clearance family in Equation~\eqref{eq:clearance_chain}, the context selected by Equation~\eqref{eq:query_scoped_selection} guarantees empirical non-degradation:
\begin{equation}
\widehat{v}_t
\left(
\widetilde{\mathcal{M}}_t
\right)
\geq
\widehat{v}_t
\left(
\mathcal{M}_t
\right).
\label{eq:query_scoped_monotonicity}
\end{equation}
If $\widehat{j}_t>0$, the selected context satisfies verified recovery:
\begin{equation}
\nu_t
\left(
\widetilde{\mathcal{M}}_t
\right)
=
1.
\label{eq:query_scoped_recovery}
\end{equation}
Furthermore, $\mathcal{C}_{t,\widehat{j}_t}$ achieves minimal cardinality among all admissible candidates attaining the maximum gain.
\end{theorem}

Because the unchanged baseline $\mathcal{C}_{t,0}=\emptyset$ is always admissible with zero gain, the selected candidate cannot yield negative empirical gain, ensuring non-degradation as detailed in Appendix C. 

\section{Experiments}
\label{sec:experiments}

We evaluate MeClear using Kimi-k2.6 as the task agent and Qwen3.6-Flash as the judge evaluator across 745 causally verified test cases containing 1,115 fault records, synthesized from 368 clean queries over ten LoCoMo conversations. Under a hybrid retrieval budget $K=5$, memory faults span direct conflicts with $n=375$ and $|M|=1$, redundant conflicts with $n=285$ and $|M|=2$, and joint interactions with $n=85$ and $|M|=2$. We compare MeClear configured with $\tau=\kappa=0.05$ and sampling budget $L=16$ against Leave-One-Out, ContextCite, and ProxySPEX across four metrics: Target Recall at $|M|$, Complete Set Recall, Exact Set Match, and Binary Task Recovery. Detailed dataset construction, fault verification protocols, and metric definitions are provided in Appendix~D.

\subsection{Overall Attribution Accuracy and Task Recovery}
\label{subsec:overall_performance}

MeClear outperforms all baselines in attribution precision and downstream task recovery. As shown in Figure~\ref{fig:overall_performance}, MeClear achieves $85.9\%$ memory micro-recall, $47.0\%$ exact fault set identification, and $82.3\%$ binary task recovery, whereas LOO suffers from local blind spots, obtaining only $38.3\%$ recall and $56.8\%$ recovery. While ContextCite and ProxySPEX yield high recall, their inability to isolate complete harmful coalitions limits exact identification to $43.5\%$ and $54.2\%$. Paired comparisons across $n=745$ samples in Figure~\ref{fig:paired_advantage} confirm MeClear's dominance, yielding $212$ wins against $22$ losses over LOO and significant recovery gains of $+25.5\%$ over LOO, $+5.0\%$ over ProxySPEX, and $+38.7\%$ over Qwen3.5-Plus.

\subsection{Performance Across Structural Fault Mechanisms}
\label{subsec:mechanism_breakdown}

\begin{figure}[t]
    \centering
    \includegraphics[width=1\linewidth]{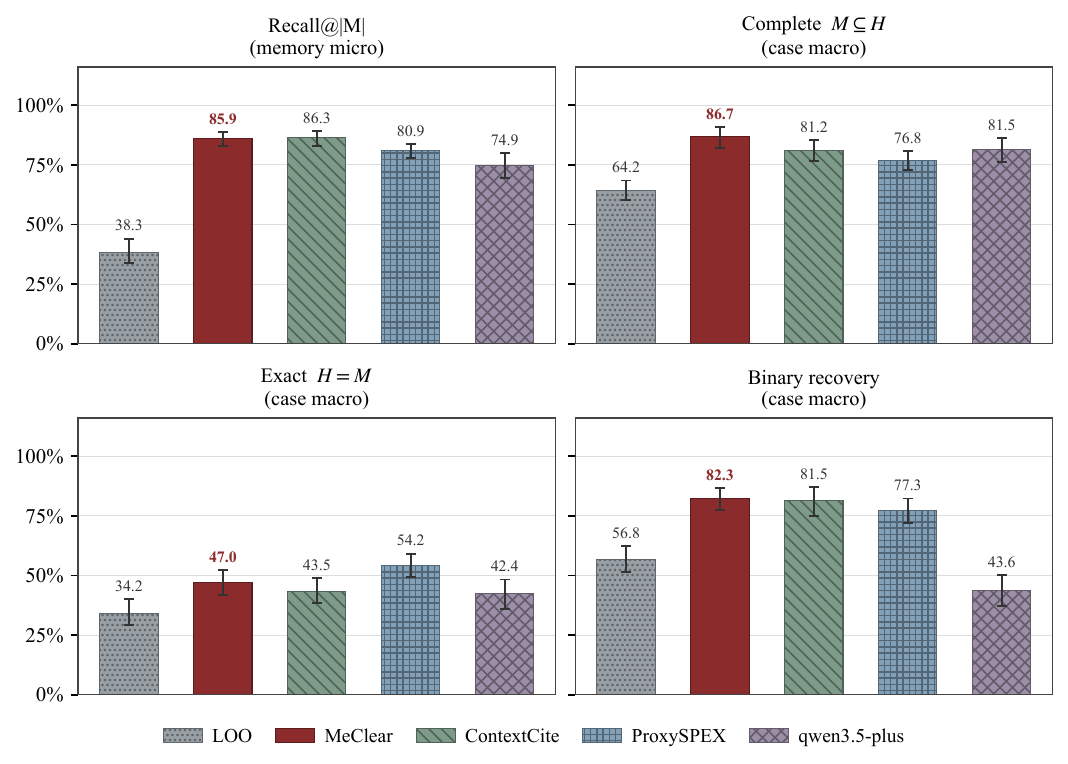}
    \caption{Overall attribution accuracy and task recovery across evaluation cases compared to baselines.}
    \label{fig:overall_performance}
\end{figure}

\begin{figure}[t]
    \centering
    \includegraphics[width=1\linewidth]{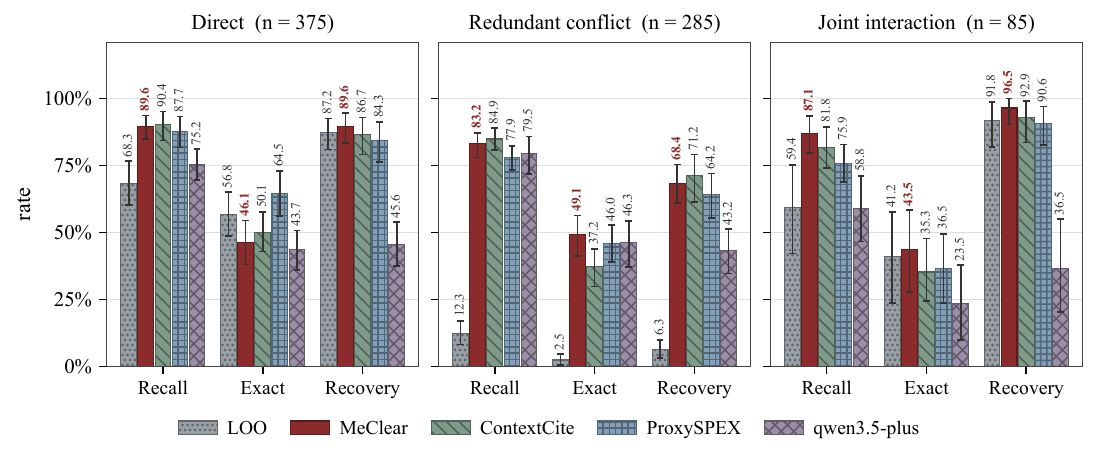}
    \caption{Performance across direct conflict, $n=375$, redundant masking, $n=285$, and joint interaction, $n=85$.}
    \label{fig:mechanism_breakdown}
\end{figure}

\begin{figure}[t]
    \centering
    \includegraphics[width=1\linewidth]{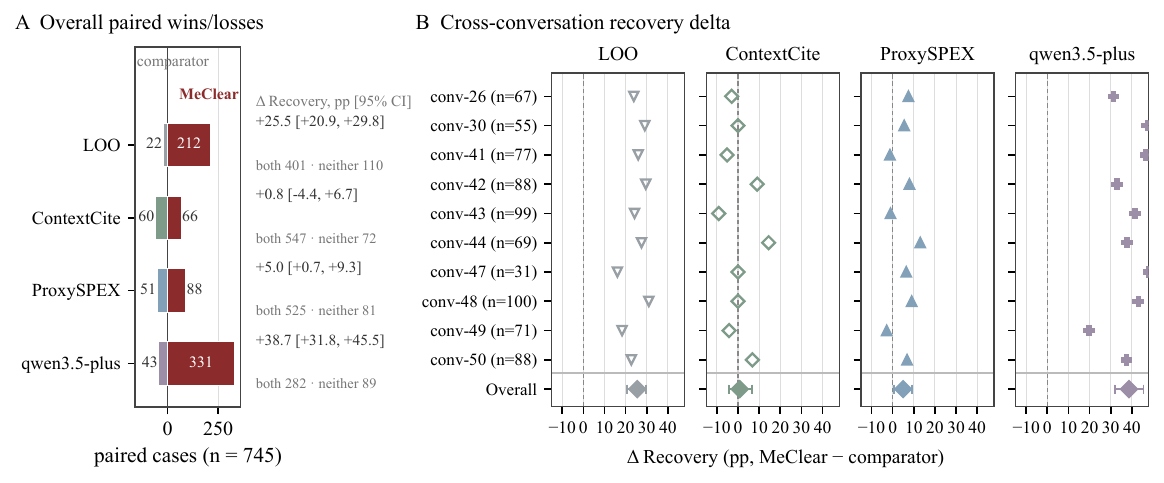}
    \caption{Paired comparative advantage and recovery gains. Panel A: head-to-head win/loss statistics across $n=745$ cases; Panel B: recovery differentials $\Delta\text{Recovery}$ with 95\% confidence intervals.}
    \label{fig:paired_advantage}
\end{figure}

\begin{figure}[t]
    \centering
    \includegraphics[width=1\linewidth]{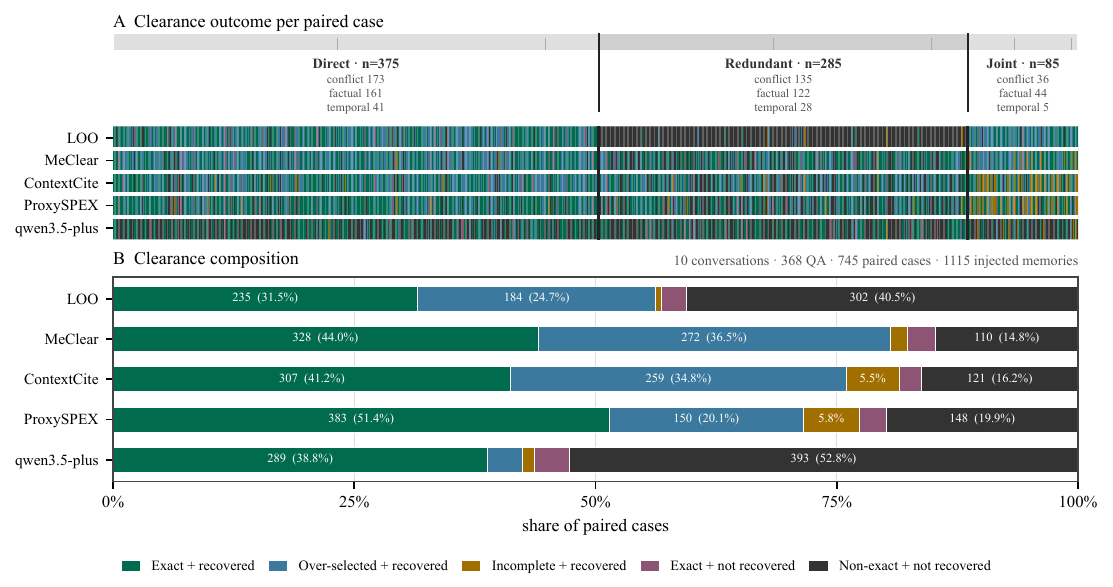}
    \caption{Query-scoped memory clearance and context composition. Panel A: clearance profiles by conflict type; Panel B: context composition across exact recovery, over-selection, and non-recovery.}
    \label{fig:clearance_outcomes}
\end{figure}

Stratifying evaluation across direct conflicts with $n=375$, redundant masking with $n=285$, and joint faults with $n=85$ reveals that performance gaps stem from multi-memory dependencies, as shown in Figure~\ref{fig:mechanism_breakdown}. In direct conflicts, LOO achieves $68.3\%$ recall and $87.2\%$ recovery, while MeClear reaches $83.2\%$ recall and $89.6\%$ recovery. Under redundant conflicts, local substitution causes LOO to collapse to $12.3\%$ recall and $6.3\%$ recovery. By evaluating marginal contributions across permutations, MeClear overcomes local masking, sustaining $83.2\%$ recall, $49.1\%$ exact identification, and $68.4\%$ recovery. Under joint non-additive interactions where LOO fails with $0.0\%$ exact match, MeClear achieves $87.1\%$ recall and $91.8\%$ recovery, demonstrating robust handling of higher-order evidence dependencies.

\subsection{Case Studies on Redundant Masking Resolution}
\label{subsec:case_studies}

Table~\ref{tab:meclear-redundant-cases} illustrates how MeClear resolves local masking in \textbf{representative} redundant fault scenarios. Under plateau ($\mathbf{c} = [1.00, 0.50, 0.50, 0.50]$) and floor ($\mathbf{c} = [1.00, 0.00, 0.00, 0.00]$) masking signatures, LOO fails by returning $\varnothing$ or removing benign records because single-fault removals leave substitute harm active. In contrast, MeClear accurately attributes negative cooperative contributions (e.g., $\widehat{\psi}(F_1) = -0.094$ and $\widehat{\psi}(F_2) = -0.219$ in Case A), successfully isolating the exact harmful set $\mathcal{M}$ and restoring the full task utility.

\subsection{Query-Scoped Clearance and Context Composition}
\label{subsec:clearance_composition}

Figure~\ref{fig:clearance_outcomes} confirms that MeClear \textbf{consistently} maximizes task recovery while minimizing context distortion. MeClear achieves an $80.5\%$ overall recovery rate ($44.0\%$ exact set and $36.5\%$ harmless over-selection) with only $14.8\%$ unrecovered failures. Conversely, LOO yields $40.5\%$ failures due to missed redundant faults, and direct baseline LLM filtering via Qwen3.5-Plus results in $52.8\%$ failures from over-truncation or hallucinated deletion. Thus, MeClear's verified selection rule effectively isolates minimal harmful coalitions while preserving the beneficial task context.

\section{Conclusion}
Unchecked memory accumulation threatens long-horizon LLM agents by turning external memory into an execution-degrading bottleneck of conflicting evidence. To break this impasse, we propose MeClear, a framework combining permutation-sampled Shapley attribution with a query-scoped verified clearance gateway. Extensive evaluations demonstrate that MeClear achieves an 85.9\% target recall and an 82.3\% task recovery rate, outperforming Leave-One-Out baselines by 25.5 percentage points. Ultimately, this work establishes coalition-aware context clearance as a foundational pillar for safe, non-destructive agent memory maintenance.

\bibliography{aaai2027}

\clearpage
\appendix

\section{Appendix}

Due to the strict page limit of the main paper, we provide supplementary theoretical analyses and experimental details in the appendix. Appendix~A proves the redundant-harm blind spot theorem and analyzes why LOO fails under redundant harmful memories. Appendix~B establishes the finite-sample guarantee of cooperative attribution estimation. Appendix~C provides the theoretical guarantee for query-scoped verified clearance. Appendix~D reports the complete experimental configurations, baseline settings, evaluation protocols, and additional analyses.

\section{Appendix A: Redundant-Harm Blind Spot}
\label{app:redundant_harm}

This appendix proves Theorem~\ref{thm:redundant_harm} and explains why a single-record Leave-One-Out (LOO) test can miss redundant harmful memories. Throughout Appendix~A, the task step $t$ and the retrieved context $\mathcal{M}_t$ are fixed. For the theoretical analysis, let
$d_{t,i}=v_t(\mathcal{M}_t)-v_t(\mathcal{M}_t\setminus\{m_i\})$
denote the exact counterpart of the empirical LOO effect
$\widehat d_{t,i}$ in Equation~\eqref{eq:local_effect}.

\subsection{A.1 Proof of Theorem~\ref{thm:redundant_harm}}

Assume that $m_i$ and $m_j$ satisfy the redundant-harm model in
Equation~\eqref{eq:redundant_game}:
\begin{equation}
\begin{aligned}
&v_t(S) ={} \\
& u_t\!\left(S\setminus\{m_i,m_j\}\right) - \Delta_t \mathbf{1}\!\left[S\cap\{m_i,m_j\}\neq\emptyset\right], \Delta_t>0,
\end{aligned}
\label{eq:app_redundant_game}
\end{equation}
where $u_t$ does not depend on $m_i$ or $m_j$, and both memories belong to $\mathcal{M}_t$.

\begin{proof}
Because both $m_i$ and $m_j$ are present in $\mathcal{M}_t$, the penalty in Equation~\eqref{eq:app_redundant_game} is active and
\begin{equation}
v_t(\mathcal{M}_t)
=
u_t\!\left(
\mathcal{M}_t\setminus\{m_i,m_j\}
\right)
-\Delta_t,
\label{eq:app_redundant_full}
\end{equation}
Removing $m_i$ leaves $m_j$ in the context, so the penalty remains:
\begin{equation}
v_t\!\left(
\mathcal{M}_t\setminus\{m_i\}
\right)
=
u_t\!\left(
\mathcal{M}_t\setminus\{m_i,m_j\}
\right)
-\Delta_t.
\label{eq:app_redundant_remove_i}
\end{equation}

Hence $d_{t,i}=0$. By symmetry, $d_{t,j}=0$.
We next consider the cooperative contribution. Shapley-based valuation measures a player's contribution by averaging its marginal effect over different coalitions
\cite{ghorbani2019datashapley,jia2019towards}.
Let $\pi$ be a uniformly random permutation of $\mathcal{M}_t$, and let $P_i(\pi)$ contain the memories appearing before $m_i$. Define
\begin{equation}
X_{t,i}(\pi)
=
v_t\!\left(P_i(\pi)\cup\{m_i\}\right)
-
v_t\!\left(P_i(\pi)\right).
\label{eq:app_marginal_contribution}
\end{equation}

If $m_i$ appears before $m_j$, then $m_j\notin P_i(\pi)$. Adding $m_i$ activates the penalty for the first time, so
$X_{t,i}(\pi)=-\Delta_t$.
If $m_j$ appears before $m_i$, the penalty is already active and
$X_{t,i}(\pi)=0$.
The two relative orders are equally likely, each with probability $1/2$. Therefore, using Equation~\eqref{eq:permutation_contribution},
\begin{equation}
\psi_{t,i}
=
\mathbb{E}_{\pi}[X_{t,i}(\pi)]
=
-\frac{\Delta_t}{2}.
\label{eq:app_psi_i}
\end{equation}
The argument gives $\psi_{t,j}=-\Delta_t/2$. Thus
$d_{t,i}=d_{t,j}=0$ while
$\psi_{t,i}=\psi_{t,j}=-\Delta_t/2$, proving
Theorem~\ref{thm:redundant_harm}.
\end{proof}

\subsection{A.2 Why LOO Misses Redundant Harm}

LOO and cooperative attribution evaluate a memory under different contexts. The LOO effect measures the performance change caused by removing one memory from the retrieved context. Specifically,
for memory $m_i$, it evaluates
\begin{equation}
d_{t,i}
=
v_t(\mathcal{M}_t)
-
v_t(\mathcal{M}_t\setminus\{m_i\}).
\label{eq:app_loo_context_difference}
\end{equation}

Under Equation~\eqref{eq:redundant_game}, removing $m_i$ leaves the
substitutable harmful memory $m_j$ active. Therefore, the degradation
term remains unchanged before and after the deletion:
\begin{equation}
v_t(\mathcal{M}_t)
=
v_t(\mathcal{M}_t\setminus\{m_i\}),
\label{eq:app_loo_redundant_equivalence}
\end{equation}
which directly leads to
$d_{t,i}=0$.
Cooperative attribution instead evaluates the marginal contribution of
$m_i$ over different predecessor coalitions. Following the permutation
formulation of Shapley-based valuation
\cite{ghorbani2019datashapley,jia2019towards}, the contribution of
$m_i$ is obtained by averaging
\begin{equation}
\Delta_{t,i}(\pi)
=
v_t(P_i(\pi)\cup\{m_i\})
-
v_t(P_i(\pi)),
\label{eq:app_cooperative_marginal_difference}
\end{equation}
where $P_i(\pi)$ denotes the set of memories preceding $m_i$ in permutation $\pi$. When $m_j$ is not included in $P_i(\pi)$, the marginal contribution of $m_i$ reflects its harmful effect because no redundant memory provides the same evidence. When $m_j$ is already included, the harmful evidence is already present and the additional contribution of $m_i$ is masked by redundancy. Therefore, coalition-level marginal evaluation can reveal harmful contributions that remain invisible under the single-memory deletion test. The result should be interpreted narrowly. Theorem~\ref{thm:redundant_harm} only establishes that a zero LOO effect does not exclude harmful cooperative contribution under redundant harmful memories. It does not imply that every memory with a zero or small LOO effect is harmful. Accordingly, MeClear interprets the LOO profile using tolerance $\kappa$: $\widehat d_{t,i}<-\kappa$ provides local evidence of harm, $|\widehat d_{t,i}|\leq\kappa$ is treated as locally inconclusive, and $\widehat d_{t,i}>\kappa$ provides local evidence of benefit. This local profile does not replace cooperative attribution; the operational harmful set remains determined by the sampled contribution in Equation~\eqref{eq:sampled_shapley} with harm tolerance $\tau$.

\subsection{A.3 Extension to Multiple Redundant Memories}

The same blind spot occurs for a redundant group with more than two memories.

\noindent\textbf{Corollary A.1 (Multi-Memory Redundant Harm).}
Let $\mathcal{R}=\{m_1,\ldots,m_r\}$ with $r\geq2$, and suppose
\begin{equation}
v_t(S)
=
u_t(S\setminus\mathcal{R})
-
\Delta_t
\mathbf{1}
\!\left[
S\cap\mathcal{R}\neq\emptyset
\right],
\qquad
\Delta_t>0,
\label{eq:app_multi_redundant_game}
\end{equation}
where $u_t$ does not depend on any memory in $\mathcal{R}$. If
$\mathcal{R}\subseteq\mathcal{M}_t$, then every
$m_i\in\mathcal{R}$ satisfies
$d_{t,i}=0$ and $\psi_{t,i}=-\Delta_t/r$.

\begin{proof}
Since $r\geq2$, deleting one memory $m_i$ leaves at least one member of $\mathcal{R}$ in the context. The penalty in
Equation~\eqref{eq:app_multi_redundant_game} therefore remains active, so $d_{t,i}=0$.
Along any permutation, the penalty $-\Delta_t$ is introduced exactly once: when the first member of $\mathcal{R}$ appears. Under a uniformly random permutation, each member of $\mathcal{R}$ is first with probability $1/r$. Thus the marginal contribution of $m_i$ is $-\Delta_t$ with probability $1/r$ and zero otherwise, which gives
$\psi_{t,i}=-\Delta_t/r$.
\end{proof}

Corollary~A.1 shows that the blind spot is not limited to a pair of redundant memories. Under the stated substitution model, LOO assigns zero effect to every member of the redundant group, while the cooperative allocation distributes the total penalty $-\Delta_t$ across the group.


\section{Appendix B: Finite-Sample Cooperative Estimation}
\label{app:finite_sample}

This appendix proves Proposition~\ref{prop:finite_sample_attribution}. The analysis conditions on the fixed empirical value function $\widehat v_t$ and isolates the approximation error caused by sampling permutations.

\subsection{B.1 Exact Shapley Value of the Empirical Game}

For the fixed empirical game, define
\begin{equation}
\widetilde{\psi}_{t,i}
=
\frac{1}{|\Pi_t|}
\sum_{\pi\in\Pi_t}
\left[
\widehat v_t
\!\left(
P_i(\pi)\cup\{m_i\}
\right)
-
\widehat v_t
\!\left(
P_i(\pi)
\right)
\right].
\label{eq:app_empirical_exact_shapley}
\end{equation}

Conditional on $\widehat v_t$, this quantity is fixed. The use of sampled permutations to approximate Shapley values follows the standard permutation view used in efficient Shapley estimation
\cite{ghorbani2019datashapley,jia2019towards}.
Equation~\eqref{eq:app_empirical_exact_shapley} is equivalent to
\begin{equation}
\begin{aligned}
&\widetilde{\psi}_{t,i} = \\
&\sum_{S\subseteq\mathcal{M}_t\setminus\{m_i\}}
\frac{|S|!(K-|S|-1)!}{K!}
\left[
\widehat v_t(S\cup\{m_i\})
-
\widehat v_t(S)
\right].
\end{aligned}
\label{eq:app_empirical_coalition_shapley}
\end{equation}

To see this, fix
$S\subseteq\mathcal{M}_t\setminus\{m_i\}$ with $|S|=s$.
There are $s!(K-s-1)!$ permutations for which $S$ is exactly the predecessor set of $m_i$: the $s$ memories in $S$ can appear before $m_i$ in any order, and the remaining $K-s-1$ memories can appear after $m_i$ in any order. Dividing by the total number $K!$ of permutations gives the coefficient in Equation~\eqref{eq:app_empirical_coalition_shapley}.
The empirical Shapley values also satisfy the efficiency property
\begin{equation}
\sum_{i=1}^{K}\widetilde{\psi}_{t,i}
=
\widehat v_t(\mathcal{M}_t)
-
\widehat v_t(\emptyset).
\label{eq:app_empirical_efficiency}
\end{equation}

\begin{proof}
Fix a permutation
$\pi=(m_{\pi_1},\ldots,m_{\pi_K})$
and define the prefix sets
$S_0=\emptyset$ and
$S_r=\{m_{\pi_1},\ldots,m_{\pi_r}\}$.
The marginal contributions along this permutation telescope:
\begin{equation}
\sum_{r=1}^{K}
\left[
\widehat v_t(S_r)-\widehat v_t(S_{r-1})
\right]
=
\widehat v_t(\mathcal{M}_t)-\widehat v_t(\emptyset).
\label{eq:app_telescoping}
\end{equation}
Averaging the left-hand side over all permutations gives
$\sum_i\widetilde{\psi}_{t,i}$, while the right-hand side is unchanged. This proves Equation~\eqref{eq:app_empirical_efficiency}. Replacing
$\widehat v_t$ with $v_t$ gives the population identity stated after
Equation~\eqref{eq:memory_contribution}.
\end{proof}

\subsection{B.2 Unbiasedness and Uniform Concentration}

For a sampled permutation $\pi^{(\ell)}$, define
\begin{equation}
X_{t,i}^{(\ell)}
=
\widehat v_t
\!\left(
P_i(\pi^{(\ell)})\cup\{m_i\}
\right)
-
\widehat v_t
\!\left(
P_i(\pi^{(\ell)})
\right).
\label{eq:app_sample_marginal}
\end{equation}

Then Equation~\eqref{eq:sampled_shapley} can be written as
$\widehat{\psi}_{t,i}^{(L)}
=L^{-1}\sum_{\ell=1}^{L}X_{t,i}^{(\ell)}$.
For each fixed $i$, the variables
$X_{t,i}^{(1)},\ldots,X_{t,i}^{(L)}$
are independent and identically distributed conditional on
$\widehat v_t$, because the permutations are sampled independently and uniformly. By Equation~\eqref{eq:app_empirical_exact_shapley},
$\mathbb{E}[X_{t,i}^{(\ell)}\mid\widehat v_t]
=\widetilde{\psi}_{t,i}$.
Linearity of expectation therefore gives
\begin{equation}
\mathbb{E}
\left[
\widehat{\psi}_{t,i}^{(L)}
\mid
\widehat v_t
\right]
=
\widetilde{\psi}_{t,i}.
\label{eq:app_conditional_unbiasedness}
\end{equation}

Since the empirical game is fixed in Proposition~\ref{prop:finite_sample_attribution}, this is the unbiasedness statement in Equation~\eqref{eq:sampled_unbiasedness}.
Because $\widehat v_t(S)\in[0,1]$ for every coalition $S$, each sampled marginal satisfies
$X_{t,i}^{(\ell)}\in[-1,1]$.
Hoeffding's inequality therefore gives, for every fixed $i$ and
$\varepsilon>0$,
\begin{equation}
\Pr
\left(
\left|
\widehat{\psi}_{t,i}^{(L)}
-
\widetilde{\psi}_{t,i}
\right|
\geq
\varepsilon
\;\middle|\;
\widehat v_t
\right)
\leq
2
\exp
\left(
-\frac{L\varepsilon^2}{2}
\right).
\label{eq:app_single_hoeffding}
\end{equation}

Applying the union bound over the $K$ memories yields
\begin{equation}
\Pr
\left(
\max_{1\leq i\leq K}
\left|
\widehat{\psi}_{t,i}^{(L)}
-
\widetilde{\psi}_{t,i}
\right|
\geq
\varepsilon
\;\middle|\;
\widehat v_t
\right)
\leq
2K
\exp
\left(
-\frac{L\varepsilon^2}{2}
\right),
\label{eq:app_uniform_hoeffding}
\end{equation}
which proves Equation~\eqref{eq:sampled_uniform_bound}.
The estimates for different memories need not be independent because one sampled permutation contributes to several memories. This does not affect the proof: Hoeffding's inequality is applied separately to the $L$ independent permutation samples for each fixed memory, and the final union bound does not require independence across memories.

\subsection{B.3 Consequences for Sampling, Thresholding, and Ranking}

Equation~\eqref{eq:app_uniform_hoeffding} gives a sufficient sample size for uniform accuracy. For any $\varepsilon>0$ and
$\delta\in(0,1)$, if
\begin{equation}
L
\geq
\frac{2}{\varepsilon^2}
\log
\left(
\frac{2K}{\delta}
\right),
\label{eq:app_sample_complexity}
\end{equation}
then, with probability at least $1-\delta$,
\begin{equation}
\max_{1\leq i\leq K}
\left|
\widehat{\psi}_{t,i}^{(L)}
-
\widetilde{\psi}_{t,i}
\right|
<
\varepsilon.
\label{eq:app_uniform_accuracy}
\end{equation}

This is a sufficient worst-case bound based only on
$\widehat v_t(S)\in[0,1]$; it is not a claim that the resulting value of $L$ is necessary or optimal in practice.
We next consider the threshold in Equation~\eqref{eq:operational_harmful_set}. Define the exact harmful set of the empirical game as
$\widetilde{\mathcal{H}}_t
=\{m_i\in\mathcal{M}_t\mid\widetilde{\psi}_{t,i}<-\tau\}$,
and assume that no exact empirical contribution lies on the threshold. Let
$\gamma_t=\min_i|\widetilde{\psi}_{t,i}+\tau|>0$.
Setting $\varepsilon=\gamma_t$ in
Equation~\eqref{eq:app_sample_complexity} gives
\begin{equation}
L
\geq
\frac{2}{\gamma_t^2}
\log
\left(
\frac{2K}{\delta}
\right).
\label{eq:app_threshold_complexity}
\end{equation}

Under this condition,
$\widehat{\mathcal{H}}_t=\widetilde{\mathcal{H}}_t$
with probability at least $1-\delta$, because no estimate can cross the threshold $-\tau$ when its error is smaller than $\gamma_t$.
The same argument controls the ordering in
Equation~\eqref{eq:harmful_order}. If
$\widetilde{\psi}_{t,j}-\widetilde{\psi}_{t,i}>2\varepsilon$
and the uniform estimation error is smaller than $\varepsilon$, then
$\widehat{\psi}_{t,i}^{(L)}<\widehat{\psi}_{t,j}^{(L)}$.
Thus, pairs of memories separated by more than twice the estimation error keep the same order.
These results apply to the fixed empirical game only. Proposition~\ref{prop:finite_sample_attribution} controls the difference between
$\widehat{\psi}_{t,i}^{(L)}$ and $\widetilde{\psi}_{t,i}$.
It does not by itself control the difference between
$\widehat{\psi}_{t,i}^{(L)}$ and the population contribution
$\psi_{t,i}$. Such a statement would require an additional assumption relating $\widehat v_t$ to $v_t$. The analysis above concerns the approximation error of cooperative attribution and therefore depends on the harm tolerance $\tau$. The LOO tolerance $\kappa$ is used only to interpret the preceding local counterfactual profile and does not alter the finite-sample guarantee for the sampled Shapley estimator.


\section{Appendix C: MeClear Procedure and Clearance Guarantee}
\label{app:query_scoped_clearance}

This appendix gives the full MeClear procedure using the definitions in the main paper and proves Theorem~\ref{thm:query_scoped_clearance}. All counterfactual evaluations are performed on the same frozen context $\mathcal{M}_t$, as specified in the main method.

\begin{algorithm}[t]
\caption{MeClear: Cooperative Attribution and Verified Clearance}
\label{alg:meclear}
\textbf{Input}: Query $q_t$, memory bank $\mathcal{B}_t$, retriever $\rho$,
evaluator $\widehat v_t$, recovery predicate $\nu_t$\\
\textbf{Parameters}: Retrieval budget $K$, permutation budget $L$,
screening tolerance $\kappa$, harm tolerance $\tau$\\
\textbf{Output}: Cleared context $\widetilde{\mathcal{M}}_t$

\begin{algorithmic}[1]

\STATE Retrieve and freeze $\mathcal{M}_t$ using Eq.~\eqref{eq:active_context_retrieval}
\STATE Evaluate and cache $\widehat v_t(\mathcal{M}_t)$
\STATE Compute LOO effects $\{\widehat d_{t,i}\}$ using Eq.~\eqref{eq:local_effect}
\STATE Classify the local LOO effects using tolerance $\kappa$
\STATE Estimate $\{\widehat{\psi}_{t,i}^{(L)}\}$ using Eq.~\eqref{eq:sampled_shapley}
\STATE Construct $\widehat{\mathcal H}_t$ using Eq.~\eqref{eq:operational_harmful_set}
\STATE Order $\widehat{\mathcal H}_t$ using Eq.~\eqref{eq:harmful_order}
\STATE Set $h_t\leftarrow|\widehat{\mathcal H}_t|$ and
$\mathcal{J}_t\leftarrow\{0\}$

\FOR{$j=1$ to $h_t$}
    \STATE Construct $\mathcal{C}_{t,j}$ using Eq.~\eqref{eq:clearance_chain}
    \STATE Compute $\widehat g_t(j)$ using Eq.~\eqref{eq:query_scoped_gain}
    \IF{$\widehat g_t(j)>0$ and
    $\nu_t(\mathcal{M}_t\setminus\mathcal{C}_{t,j})=1$}
        \STATE $\mathcal{J}_t\leftarrow\mathcal{J}_t\cup\{j\}$
    \ENDIF
\ENDFOR

\STATE Select $\widehat j_t$ using Eq.~\eqref{eq:query_scoped_selection}
\STATE $\widetilde{\mathcal{M}}_t
\leftarrow
\mathcal{M}_t\setminus\mathcal{C}_{t,\widehat j_t}$
\STATE \textbf{return} $\widetilde{\mathcal{M}}_t$

\end{algorithmic}
\end{algorithm}

\subsection{C.1 Proof of Theorem~\ref{thm:query_scoped_clearance}}
Algorithm~\ref{alg:meclear} follows the four stages described in the main paper: retrieval, local screening, cooperative attribution, and verified clearance. The LOO profile is interpreted using screening tolerance $\kappa$: strongly negative local effects provide direct evidence of harm, near-zero effects remain inconclusive, and strongly positive effects indicate local benefit. The final operational harmful set is determined by cooperative contributions with harm tolerance $\tau$. The interaction score in Equation~\eqref{eq:sampled_interaction} is used only for structural analysis and does not affect harmful-set construction or clearance decisions. Repeated coalition values may be cached without changing the sampled estimator in Equation~\eqref{eq:sampled_shapley}.
Recall that Equation~\eqref{eq:admissible_set} always includes the baseline index $0$, while every positive index must have both positive empirical gain and verified recovery. Equation~\eqref{eq:query_scoped_selection} selects the maximum-gain admissible index and uses the smallest index to break ties.

\begin{proof}
Since $0\in\mathcal{J}_t$ and $\mathcal{C}_{t,0}=\emptyset$, Equation~\eqref{eq:query_scoped_gain} gives
$\widehat g_t(0)=0$.
Therefore
$\max_{j\in\mathcal{J}_t}\widehat g_t(j)\geq0$.
Because $\widehat j_t$ is selected from this set of maximizers,
$\widehat g_t(\widehat j_t)\geq0$.
With
$\widetilde{\mathcal{M}}_t
=\mathcal{M}_t\setminus\mathcal{C}_{t,\widehat j_t}$,
we obtain
\begin{equation}
\widehat v_t(\widetilde{\mathcal{M}}_t)
\geq
\widehat v_t(\mathcal{M}_t),
\end{equation}
which proves Equation~\eqref{eq:query_scoped_monotonicity}.
Now suppose $\widehat j_t>0$. Since
$\widehat j_t\in\mathcal{J}_t$, the definition of
$\mathcal{J}_t$ requires
$\nu_t(\mathcal{M}_t\setminus\mathcal{C}_{t,\widehat j_t})=1$.
Hence
$\nu_t(\widetilde{\mathcal{M}}_t)=1$,
which proves Equation~\eqref{eq:query_scoped_recovery}.
Finally, let
\begin{equation}
\mathcal{J}_t^{\max}
=
\operatorname*{arg\,max}_{j\in\mathcal{J}_t}
\widehat g_t(j).
\label{eq:app_maximizer_indices}
\end{equation}
By Equation~\eqref{eq:query_scoped_selection},
$\widehat j_t=\min\mathcal{J}_t^{\max}$.
The nested construction in Equation~\eqref{eq:clearance_chain} satisfies
$|\mathcal{C}_{t,j}|=j$.
Thus, for every $j\in\mathcal{J}_t^{\max}$,
$|\mathcal{C}_{t,\widehat j_t}|
=\widehat j_t\leq j=|\mathcal{C}_{t,j}|$.
Therefore
$\mathcal{C}_{t,\widehat j_t}$
has minimum cardinality among the admissible candidates that attain the maximum empirical gain.
\end{proof}

The baseline also gives a simple fallback. If no positive index satisfies the gain and recovery conditions, then
$\mathcal{J}_t=\{0\}$ and MeClear returns the unchanged context. If at least one positive admissible candidate exists, its gain is strictly larger than the baseline gain, so the selected index is positive and the returned context satisfies the recovery predicate.

\subsection{C.2 Scope of the Guarantee and Evaluator Error}
The ideal objective in Equation~\eqref{eq:clearance_objective} considers all subsets of the harmful set, whereas Theorem~\ref{thm:query_scoped_clearance} is established over the nested family in Equation~\eqref{eq:clearance_chain} with $h_t+1$ candidates. It guarantees empirical non-degradation, verified recovery for positive admissible selections, and minimum cardinality among maximum-gain candidates within this family, but does not imply global optimality over all subsets of $\widehat{\mathcal{H}}_t$. The guarantee is conditional on the operational harmful set produced by the preceding attribution stage and does not assert that $\widehat{\mathcal{H}}_t$ equals the unknown exact harmful set $\mathcal{H}_t$. Global optimality would additionally require that a minimum-cardinality unrestricted maximizer be represented in the nested family and, when nonempty, satisfy the recovery predicate.
Suppose that, for the baseline and every candidate evaluated by MeClear,
\begin{equation}
\left|
\widehat v_t(S)-v_t(S)
\right|
\leq
\eta_t
\label{eq:app_uniform_evaluator_error}
\end{equation}
for some $\eta_t\geq0$. Let
$g_t(j)
=v_t(\mathcal{M}_t\setminus\mathcal{C}_{t,j})-v_t(\mathcal{M}_t)$.
Then the triangle inequality gives
\begin{equation}
\left|
\widehat g_t(j)-g_t(j)
\right|
\leq
2\eta_t.
\label{eq:app_gain_discrepancy}
\end{equation}

Since Theorem~\ref{thm:query_scoped_clearance} gives
$\widehat g_t(\widehat j_t)\geq0$, it follows that
$g_t(\widehat j_t)\geq-2\eta_t$.
Moreover, if
$\widehat g_t(\widehat j_t)>2\eta_t$,
then $g_t(\widehat j_t)>0$.
This last statement is conditional on
Equation~\eqref{eq:app_uniform_evaluator_error}. It does not turn the empirical guarantee into an unconditional population guarantee; it only shows how a known uniform evaluator error transfers to the selected gain. 

\begin{table}[t]
\centering
\small
\setlength{\tabcolsep}{4.2pt}
\renewcommand{\arraystretch}{1.10}
\begin{tabular}{l l}
\toprule
\textbf{Setting} & \textbf{Configuration} \\
\midrule
Dataset & LoCoMo, categories 1--4 \\
Conversations & 10 \\
Clean queries & 368 \\
Verified cases & 745 \\
Fault records & 1,115 \\
Direct conflicts & 375 \\
Redundant conflicts & 285 \\
Joint interactions & 85 \\
Memory framework & Mem0 2.0.12 + Qdrant \\
Memory records & 11,302 \\
Retrieval budget $K$ & 5 \\
\bottomrule
\end{tabular}
\caption{Dataset statistics and evaluation configuration.}
\label{tab:app_dataset_config}
\end{table}

\begin{table}[t]
\centering
\small
\renewcommand{\arraystretch}{1.12}

\textbf{Panel A: Cohort construction and paired-case funnel}\\[3pt]
\begin{tabular}{@{}llrr@{}}
\toprule
Stage & Unit & $N$ & Retention \\
\midrule
LoCoMo QA & QA & 1,540 & -- \\
Stable clean-correct QA & QA & 624 & 40.5\% \\
Injection proposals & Proposal & 7,986 & -- \\
Causal-valid proposals & Proposal & 1,178 & 14.8\% \\
Selected cases & Case & 747 & 63.4\% \\
\textbf{Final paired cases} & Paired case & \textbf{745} & \textbf{99.7\%} \\
\bottomrule
\end{tabular}

\vspace{8pt}

\textbf{Panel B: Final cohort by fault and interaction}\\[3pt]
\begin{tabular}{@{}lrrrr@{}}
\toprule
Fault type & Direct & Redundant & Joint & All \\
\midrule
Explicit conflict & 173 & 135 & 36 & 344 \\
Precise factual & 161 & 122 & 44 & 327 \\
Precise temporal & 41 & 28 & 5 & 74 \\
\midrule
\textbf{All faults} & \textbf{375} & \textbf{285} & \textbf{85} & \textbf{745} \\
\bottomrule
\end{tabular}

\caption{Cohort construction and final paired-case coverage. Panel A summarizes the main data-selection milestones, and Panel B reports the final 745 cases by interaction structure and fault type.}
\label{tab:supplement_cohort}
\end{table}

\begin{figure*}[t]
    \centering
    \includegraphics[width=\linewidth]{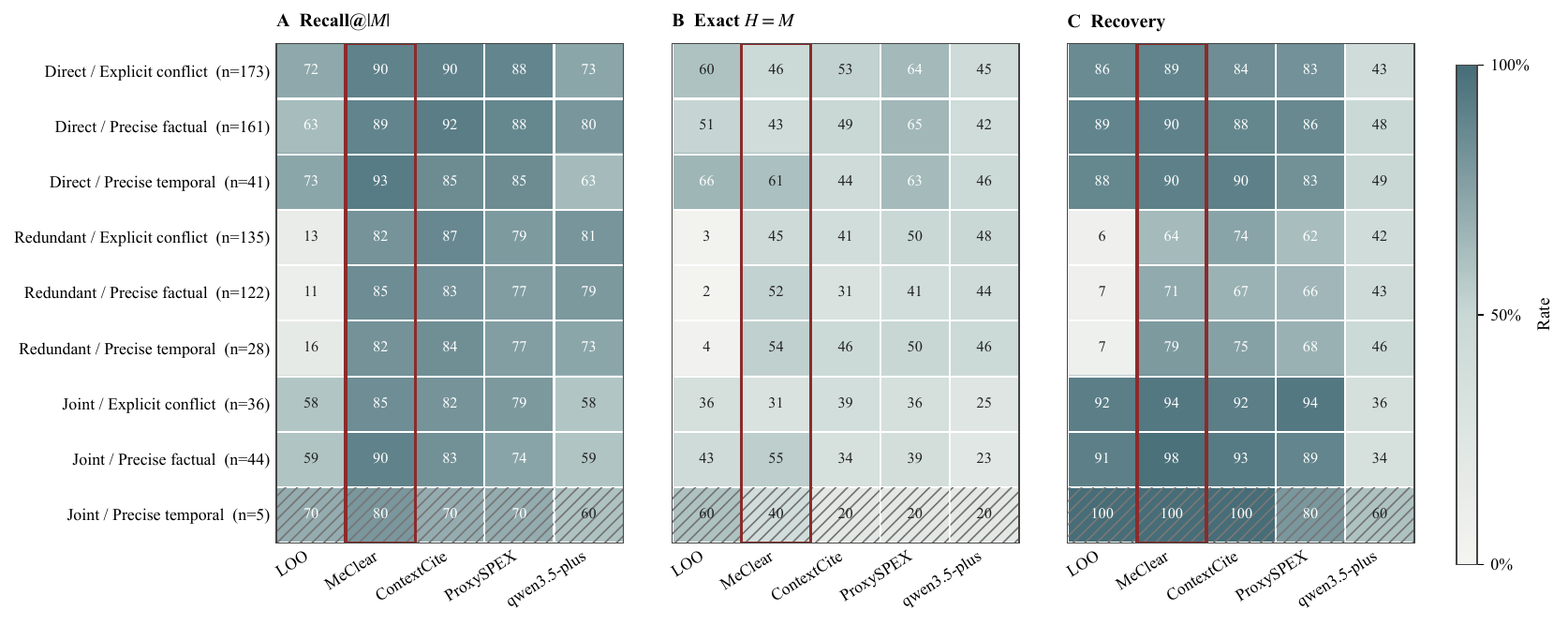}
\caption{Interaction-by-fault performance heatmap for five memory-clearance methods. Panels show target-micro Recall@$|M|$, case-macro Exact Set Match, and Binary Task Recovery across nine interaction-by-fault cells. Cell annotations report percentages and case counts; the MeClear column is outlined, and cells with $n<10$ are hatched. All panels share a 0--100\% scale with consistent visual encoding for comparison across fault structures.}
    \label{fig:interaction_fault_heatmap}
\end{figure*}

\begin{table}[t]
\centering
\small
\setlength{\tabcolsep}{4.0pt}
\renewcommand{\arraystretch}{1.10}
\begin{tabular}{l l}
\toprule
\textbf{Setting} & \textbf{Value} \\
\midrule
MeClear retrieval budget $K$ & 5 \\
MeClear permutation budget $L$ & 16 \\
MeClear checkpoints & 4, 8, 16 \\
MeClear screening tolerance $\kappa$ & 0.05 \\
MeClear harm tolerance $\tau$ & 0.05 \\
Task-value range & $[0,1]$ \\
Task agent & Kimi-k2.6 \\
Judge evaluator & Qwen3.6-Flash \\
Memory construction & Moonshot-v1-32k \\
Fault generation & Qwen3.7-Plus \\
LOO threshold & $\widehat d_{t,i}<-0.05$ \\
ContextCite budget / $\alpha$ & 32 / 0.01 \\
ProxySPEX budget / order & 32 / 2 \\
Direct LLM & Qwen3.5-Plus \\
Direct LLM threshold & $\geq0.5$ \\
\bottomrule
\end{tabular}
\caption{Main method, model, and baseline settings.}
\label{tab:app_parameter_config}
\end{table}

\begin{figure}[t]
    \centering
    \includegraphics[width=\linewidth]{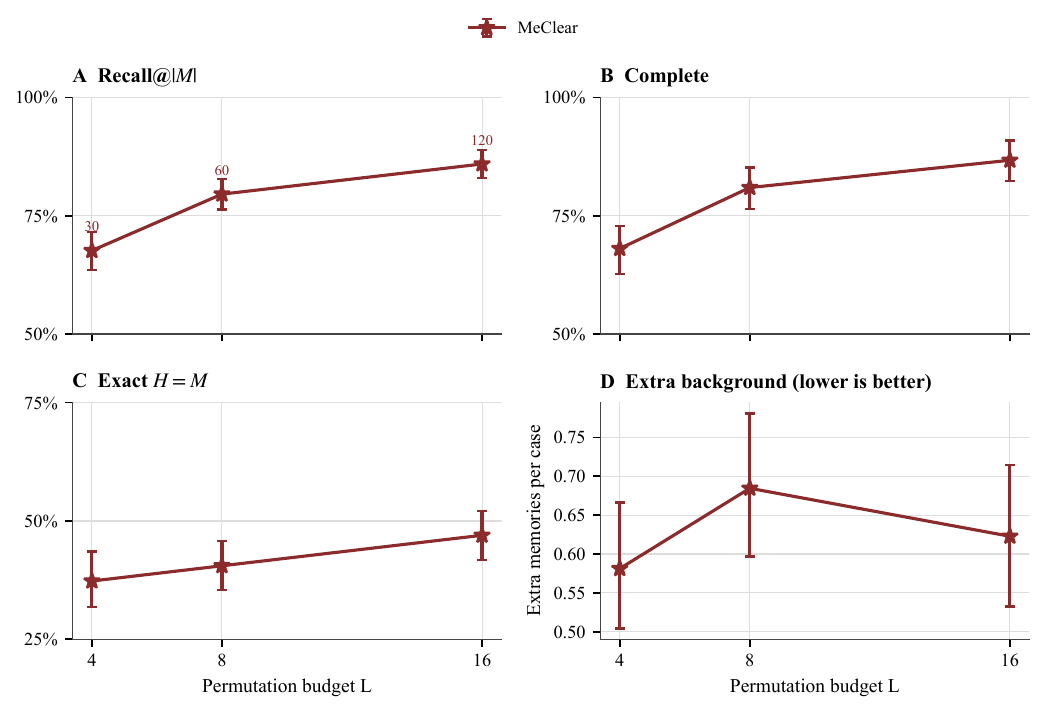}
\caption{MeClear sensitivity to permutation budgets $L=4$, $8$, and $16$ on 745 cases. Results report Recall@$|M|$, Complete Set Recall, Exact Set Match, and extra-background selection with 95\% two-level cluster-bootstrap intervals.}
    \label{fig:permutation_sensitivity}
\end{figure}
\begin{figure}[t]
    \centering
    \includegraphics[width=\linewidth]{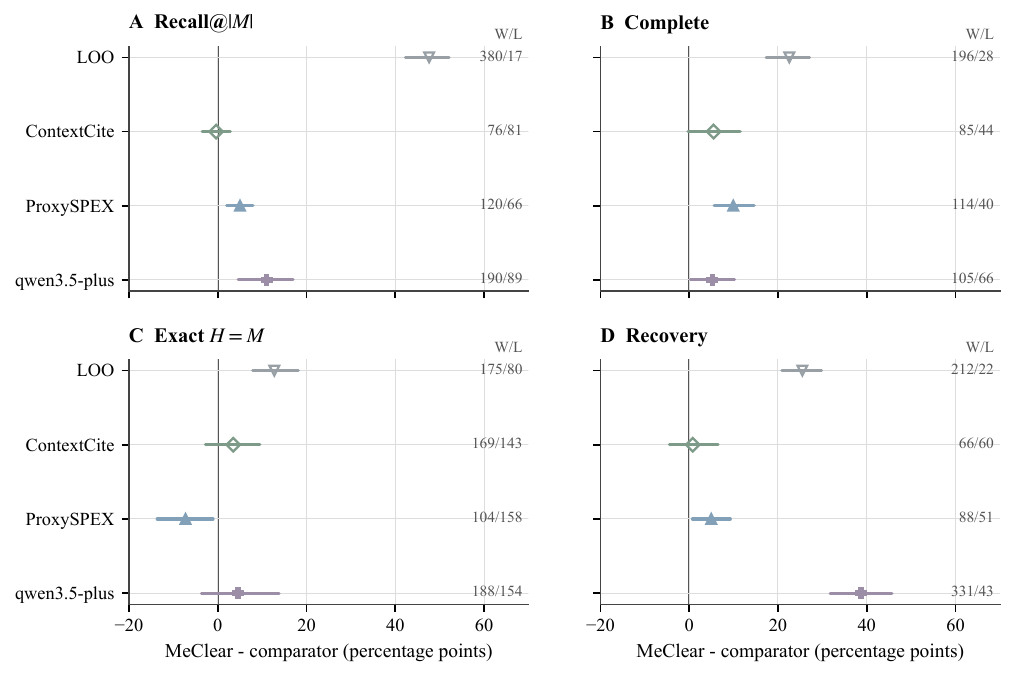}
    \caption{Paired differences between MeClear and each comparator on the common 745-case cohort. Points show MeClear-minus-comparator estimates for Recall@$|M|$, Complete Set Recall, Exact Set Match, and Binary Task Recovery; horizontal segments show 95\% two-level cluster-bootstrap intervals. The vertical line marks zero, and the right-side W/L counts report case-level positive and negative differences.}
    \label{fig:paired_differences}
\end{figure}


\section{Appendix D: Experimental Details and Parameter Settings}
\label{app:experimental_details}

Due to the strict page limit of the main paper, we provide additional experimental configuration, parameter settings, and supplementary analyses in this appendix.

\subsection{D.1 Experimental Setup}

We evaluate MeClear on LoCoMo using Kimi-k2.6 as task agent and Qwen3.6-Flash as judge evaluator. The evaluation is constructed from 368 clean queries across ten conversations and contains 745 causally verified cases with 1,115 fault records. The cohort includes 375 direct conflicts with $|M|=1$, 285 redundant conflicts with $|M|=2$, and 85 joint interactions with $|M|=2$. Two of the 747 initially selected cases were excluded before paired evaluation because complete attribution outputs could not be obtained. These exclusions resulted from evaluation failures rather than method scores, leaving 745 cases with valid outputs for compared methods under the same evaluation protocol. The interaction-by-fault composition contains 344 explicit-conflict cases, 327 precise-factual cases, and 74 precise-temporal cases. Among them, the direct-conflict group contains 173 explicit, 161 factual, and 41 temporal cases; the redundant group contains 135 explicit, 122 factual, and 28 temporal cases; and the joint-interaction group contains 36 explicit, 44 factual, and 5 temporal cases.

Memories are constructed with Mem0~2.0.12 and stored in Qdrant, resulting in 11,302 memory records. For each query, the memory system uses the retrieval budget $K=5$ specified in the main paper. The retrieved context is frozen before counterfactual attribution, and all coalition evaluations and clearance operations are performed without re-running retrieval, consistent with Equation~\eqref{eq:active_context_retrieval}. Controlled faults are retained only after behavioral verification. The clean context must remain correct in both trials, while a direct conflict must make the corrupted context incorrect in both trials. For redundant conflicts, the corrupted context must remain incorrect when either injected fault is retained alone, ensuring that each record can independently sustain the failure and mask the local effect of the other. For joint interactions, each injected record must remain harmless when evaluated alone, whereas their combination must induce task failure. These conditions provide behaviorally verified target sets for evaluating direct, redundant, and jointly expressed harmful-memory effects.

Figure~\ref{fig:interaction_fault_heatmap} reports performance across the interaction-by-fault combinations. MeClear shows its advantage under redundant conflicts, where local deletion is most susceptible to masking. Across the three redundant-fault categories, MeClear achieves $81.9$--$84.8\%$ Recall@$|M|$ compared with only $10.7$--$16.1\%$ for LOO. Exact Set Match increases from $1.6$--$3.6\%$ for LOO to $45.2$--$53.6\%$ for MeClear, while Binary Task Recovery increases from $5.9$--$7.1\%$ to $63.7$--$78.6\%$. MeClear also attains the highest observed Recall@$|M|$ across the joint-interaction cells and achieves or ties the highest Recovery in most interaction-by-fault cells. The joint precise-temporal cell contains only five cases and is therefore interpreted descriptively rather than as a stable subgroup estimate. Overall, the structural breakdown supports the role of coalition-aware attribution when harmful evidence is redundant or jointly expressed.

\subsection{D.2 Parameter and Baseline Settings}

MeClear follows the main-paper configuration with retrieval budget $K=5$, permutation budget $L=16$, and tolerance settings $\kappa=\tau=0.05$. The LOO profile uses $\kappa$ to identify locally informative or inconclusive memories, while cooperative contributions are estimated using Equation~\eqref{eq:sampled_shapley} and thresholded by $\tau$ in Equation~\eqref{eq:operational_harmful_set}. The task value is normalized to $[0,1]$ as assumed in Proposition~\ref{prop:finite_sample_attribution}.
All primary results use $L=16$, with $L\in\{4,8,16\}$ evaluated for sensitivity analysis. ContextCite and ProxySPEX use sampling budget $B=32$, with $\alpha=0.01$ for ContextCite and interaction order two for ProxySPEX. The direct LLM baseline uses Qwen3.5-Plus with threshold $0.5$, while LOO uses $\widehat d_{t,i}<-\kappa=-0.05$. Memory construction and fault generation use Moonshot-v1-32k and Qwen3.7-Plus, respectively. All paired comparisons use identical frozen retrieved contexts.
Figure~\ref{fig:permutation_sensitivity} shows that larger permutation budgets improve attribution quality. Recall@$|M|$ increases from $67.6\%$ at $L=4$ to $85.9\%$ at $L=16$, Complete Set Recall increases from $68.1\%$ to $86.7\%$, and Exact Set Match improves from $37.3\%$ to $47.0\%$. Extra-background selection remains non-monotonic ($0.581$, $0.685$, and $0.623$ for $L=4,8,16$), indicating that the improvement is not caused by excessive deletion. Therefore, $L=16$ is adopted for the main experiments. All recovery evaluations use two independent Answer/Judge trials with fixed random seeds.
\begin{figure}[t]
    \centering
    \includegraphics[width=\linewidth]{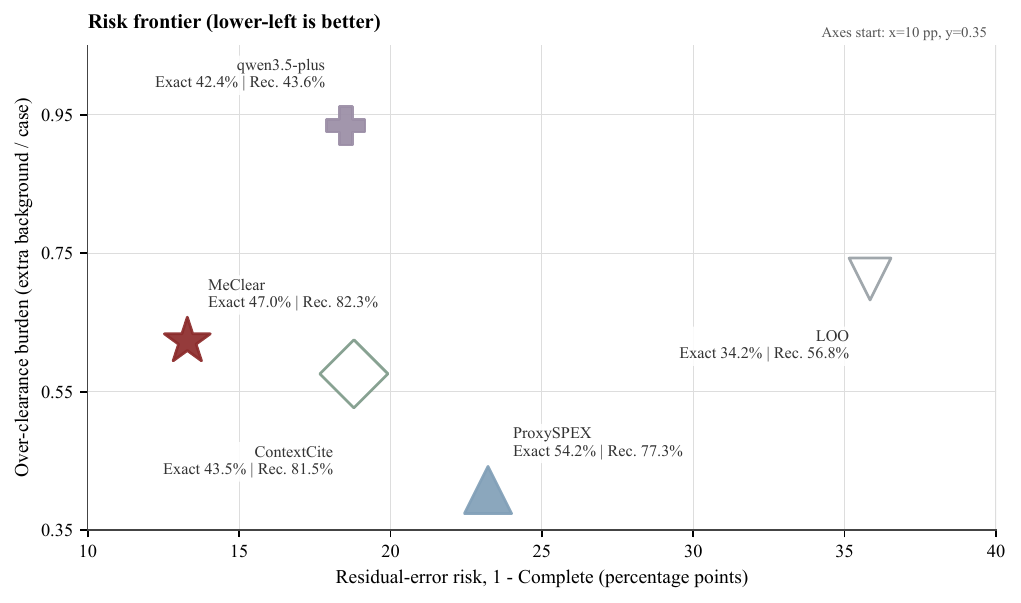}
    \caption{Risk frontier of five methods on the common cohort. The x-axis shows residual-error risk, defined as one minus Complete Set Recall, the y-axis shows the average number of extra background memories selected per case, and bubble size indicates Binary Task Recovery. Lower values are preferred on both axes.}
    \label{fig:risk_frontier}
\end{figure}

\subsection{D.3 Metrics and Supplementary Evaluation}

Following the main paper, we evaluate MeClear using four metrics: Target Recall at $|M|$, Complete Set Recall, Exact Set Match, and Binary Task Recovery. For a case with verified fault set $F$ and $|F|=m$, let $r_1,\ldots,r_m$ denote the top-$m$ ranked memories. Recall@$|M|$ is defined as
\begin{equation}
\operatorname{TR@}|M|
=
\frac{|\{r_1,\ldots,r_m\}\cap F|}{|F|}.
\label{eq:app_target_recall}
\end{equation}
The overall Recall@$|M|$ is target-micro averaged over the 1,115 verified fault records. Complete Set Recall measures whether all verified faults are selected, while Exact Set Match additionally requires no background selection. Binary Task Recovery counts only cases where both independent post-clearance trials succeed. These metrics respectively evaluate ranking quality, harmful-set completeness, selection precision, and task recovery.

\begin{figure}[t]
    \centering
    \includegraphics[width=\linewidth]{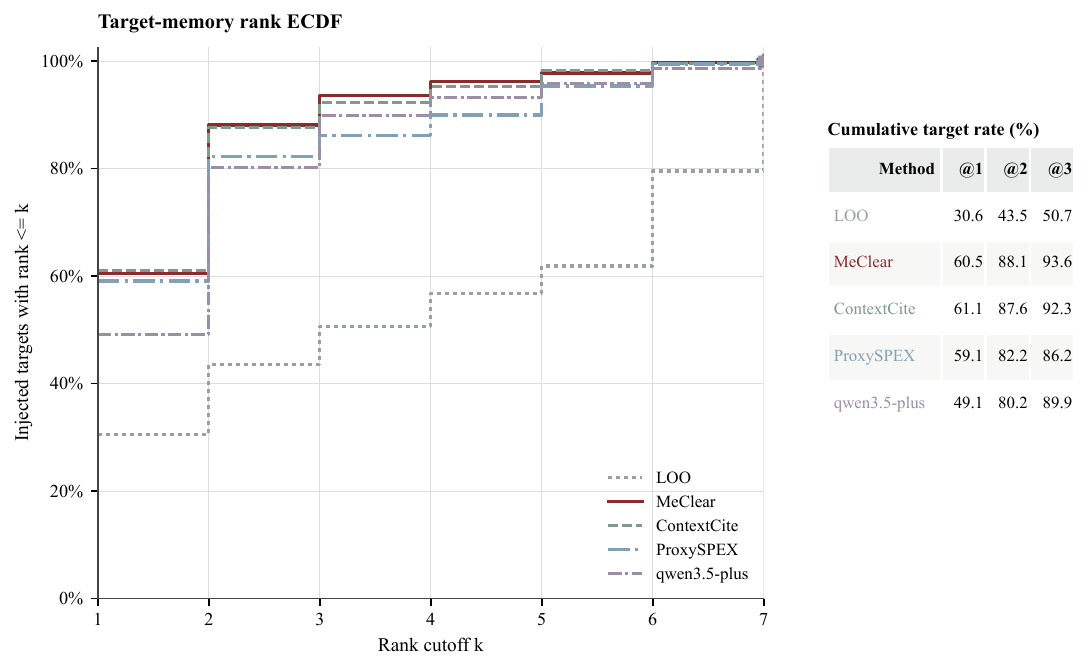}
    \caption{Target-memory rank ECDF over 1,115 injected targets. Curves show cumulative target coverage within rank $k$, with summary rates reported at $k=1$, $2$, and $3$.}
    \label{fig:target_rank_ecdf}
\end{figure}

All comparisons use the same paired cohort of 745 cases to ensure that performance differences are attributable to method behavior rather than changes in evaluation samples. We report percentile 95\% confidence intervals using 2,000 two-level cluster-bootstrap replicates, with conversations and query clusters resampled hierarchically to preserve dependencies among cases from the same source. Paired differences are computed before resampling. Figure~\ref{fig:paired_differences} shows that MeClear improves over LOO by $47.6$, $22.6$, $12.8$, and $25.5$ percentage points in Recall@$|M|$, Complete Set Recall, Exact Set Match, and Recovery, respectively. Compared with ProxySPEX, MeClear improves Recall@$|M|$, Complete Set Recall, and Recovery by $5.0$, $9.9$, and $5.0$ points, while ProxySPEX achieves higher Exact Set Match by $7.2$ points. MeClear also improves Recovery over Qwen3.5-Plus by $38.7$ points. ContextCite remains the closest comparator, with intervals overlapping zero across metrics. MeClear primarily improves harmful-memory coverage and recovery by leveraging cooperative attribution to identify interacting harmful evidence.

Figure~\ref{fig:risk_frontier} evaluates the trade-off between incomplete harmful-set removal and unnecessary context modification. MeClear achieves the lowest residual-error risk of $13.3\%$ and highest Binary Task Recovery of $82.3\%$, while selecting $0.623$ extra background memories per case. ProxySPEX selects fewer extra memories at $0.408$ per case and obtains higher Exact Set Match of $54.2\%$, but suffers higher residual risk of $23.2\%$ and lower Recovery of $77.3\%$. ContextCite achieves comparable Recovery of $81.5\%$ with $0.576$ extra selections, but retains higher residual risk of $18.8\%$. LOO and Qwen3.5-Plus show larger trade-offs, with residual risk of $35.8\%$ and extra selection of $0.934$, respectively. Figure~\ref{fig:target_rank_ecdf} further shows that MeClear provides stronger multi-rank concentration of harmful memories, reaching $88.1\%$ and $93.6\%$ target coverage at ranks $2$ and $3$. Although ContextCite slightly exceeds MeClear at rank $1$ with $61.1\%$ versus $60.5\%$, MeClear surpasses it at higher ranks. It is important for redundant and joint interactions, where effective clearance requires identifying multiple harmful memories rather than a single salient record.

\end{document}